\RequirePackage{fix-cm}

\documentclass{journal}
\usepackage{graphicx}
\usepackage[numbers]{natbib}
\usepackage{subcaption}
\usepackage{amssymb}
\usepackage{amsthm}
\usepackage{booktabs}
\usepackage{multirow}
\usepackage{bm}
\usepackage{amsmath}
\usepackage{paralist}
\usepackage{hyperref}
\usepackage{hyphenat}
\usepackage{microtype}
\usepackage{xspace}
\usepackage[dvipsnames]{xcolor}
\usepackage[linesnumbered,lined,boxed,commentsnumbered,ruled,vlined]{algorithm2e}
\usepackage{xcolor,colortbl}
\newcolumntype{a}{>{\columncolor{blue!15}}c}

\def\frocc{FROCC\xspace}

\def\RR{\mathbb{R}}

\def\lab{lab}
\def\data{dat}

\newcommand{\exsub}[2]{\mbox{E}_{#1}\left\{#2\right\}}
\newcommand{\prsub}[2]{\mbox{Pr}_{#1}\left\{#2\right\}}
\newcommand{\ind}[1]{\mathrm{1}_{#1}}
\newcommand{\prob}[1]{\mbox{Pr}\left\{#1\right\}}
\newcommand{\ex}[1]{\mbox{E}\left\{#1\right\}}

\newtheorem{theorem}{Theorem}
\newtheorem{definition}[theorem]{Definition}
\newtheorem{fact}[theorem]{Fact}
\newtheorem{lemma}[theorem]{Lemma}
\newtheorem{proposition}[theorem]{Proposition}

\newcommand{\ignore}[1]{}

\newif\ifcutlevone
\newif\ifcutlevtwo
\newif\ifcutlevthree
\cutlevonetrue
\cutlevtwotrue
\cutlevthreetrue

\def\vec#1{\boldsymbol{#1}}

\def\1#1{\mathbbm{1}_{#1}}

\newcommand\Psub[2]{\mbox{P}_{#1}(#2)}
\def\E#1#2{\mbox{E}_{#1}(#2)}

\def\auc{\texttt{AUC ROC}}

\begin{document}

\title{FROCC: Fast Random projection-based One-Class Classification}

\author{Arindam Bhattacharya         \and
        Sumanth Varambally \and
        Amitabha Bagchi \and
        Srikanta Bedathur %etc.
}

\maketitle

\begin{abstract}
    We present Fast Random projection-based One-Class Classification (FROCC), an extremely efficient method for one-class classification. Our method is based on a simple idea of transforming the training data by projecting it onto a set of random unit vectors that are chosen uniformly and independently from the unit sphere, and bounding the regions based on separation of the data. FROCC can be naturally extended with kernels. We provide a new theoretical framework to prove that that FROCC generalizes well in the sense that it is stable and has low bias for some parameter settings. FROCC achieves up to 3.1 percent points better \auc, with 1.2--67.8$\times$ speedup in training and test times over a range of state-of-the-art benchmarks including the SVM and the deep learning based models for the OCC task.

    % \keywords{one class classification \and ensemble classifier \and random projection \and stable classifier \and kernel based method
    % }
\end{abstract}
\section{Introduction}

One-class classification (OCC) has attracted a lot of attention over the years under various names such as novelty detection~\citep{Scholkopf2000}, anomaly or outlier detection~\citep{Tax2004,Sakurada_Yairi_2014}, concept learning~\citep{japkowicz1995}, and so on. Unlike the classical binary (or multi-class) classification problem, the problem addressed in OCC is much simpler: to simply identify all instances of a class by using only a (small) set of examples from that class. Therefore it is applicable in settings where only ``normal'' data is known  at training time, and the model is expected to accurately describe this normal data. After training, any data which deviates from this normality is treated as an anomaly or a novel item. Real-world applications include intrusion detection~\citep{anomaly2009,chandola2009anomaly}, medical diagnoses~\citep{towardsmedical}, fraud detection~\citep{advfraud}, etc., where the requirement is to offer not only high precision, but also the ability to scale training and testing time as the volume of (training) data increases.

Primary approaches to OCC include one-class variants of support vector machines (SVM)~\citep{Scholkopf2000}, probabilistic density estimators describing the support of normal class~\citep{cohen2008novelty,Rousseeuw_Driessen_1999}, and recently, deep-learning based methods such as auto-encoders (AE) and Generative Adversarial Networks (GAN) which try to learn the distribution of normal data samples during training~\citep{Sakurada_Yairi_2014,Perera2018LearningDF,ruff2018deep}. However, almost all of these methods are either extremely slow to train or require significant hyper-parameter tuning or both.

We present an alternative approach based on the idea of using a collection of random projections onto a number of randomly chosen 1-dimensional subspaces (vectors) resulting in a simple yet surprisingly powerful one-class classifier. Random projections down to a single dimension or to a small number of dimensions have been used in the past to speed up approximate nearest neighbor search~\citep{kleinberg-stoc:1997,ailon2006approximate,indyk2007nearest} and other problems that rely on the distance-preserving property guaranteed by the Johnson-Lindenstrauss Lemma but our problem space is different. We use random projections in an entirely novel fashion based on the following key insight: for the one-class classification problem, the preservation of distances is \emph{not} important, we only require the class boundary to be preserved. The preservation of distances is a {\em stronger} property which implies that the class boundary is preserved, but the converse need not hold. We don't need to preserve distances. All we need to do is ``view'' the training data from a ``direction'' that shows that the outlier is separated from the data.

Formalizing this intuition, we develop an algorithm for one-class classification called \frocc\ (\underline{F}ast   \underline{R}andom projection-based \underline{O}ne-\underline{C}lass \underline{C}lassification) that essentially applies a simplified envelope method to random projections of training data to provide very strong classifier. \frocc\  also comes with a theoretically provable generalization property: the method is stable and has low bias.

At a high level, \frocc works as follows: given only positively labelled data points, we select a set of random vectors from the unit sphere, and project all the training data points onto these random vectors. We compactly represent all the training data projections along each vector using one (or more) intervals, based on the separation parameter \(\varepsilon\). If a test data point's projection with the random vector is within the interval spanned by {\em all} the training data projections, we say it is part of the normal class, otherwise it is an outlier. The parameter \(\varepsilon\) breaks the single interval used to represent the random vector projection into \emph{dense} clusters with a density lower-bound $\frac{1}{\varepsilon}$, resulting in a collection of intervals along each vector.

The simplicity of the proposed model provides us with a massive advantage in terms of computational cost -- during both training and testing -- over all the existing methods ranging from SVMs to deep neural network models. We achieve 1--16$\times$ speedup for non-deep baselines like One Class SVM (OCSVM)~\cite{Scholkopf2000} and Isolation forests. For state-of-the-art deep learning based approaches like Deep-SVDD and Auto Encoder we achieve a speedup to the tune of 45$\times$ -- 68$\times$.  At the same time, the average area under the ROC curve is 0.3 — 3.1 percent points \emph{better than} the state-of-the-art baseline methods on six of the eleven benchmark datasets we used in our empirical study.

%\sbcomment{Need to write here about the speed of classification here, before talking about kernel methods}.

Like SVM our methods can be seamlessly extended to use kernels. These extensions can take advantage of kernels that may be suitable for certain data distributions, without changing the overall one-class decision framework of \frocc.
\begin{figure}[tbp]
    \resizebox{\columnwidth}{!}{
        \footnotesize
        \begin{tabular}{cccc}
            \toprule
            \multicolumn{4}{c}{Multi-modal Gaussian Data}                                                                                                                                                                                                                  \\
            ROC = 0.816                                                                                    & ROC = 0.824                                                                                  & ROC = 0.968                     & ROC = 0.836                  \\
            \includegraphics[width=0.24\columnwidth]{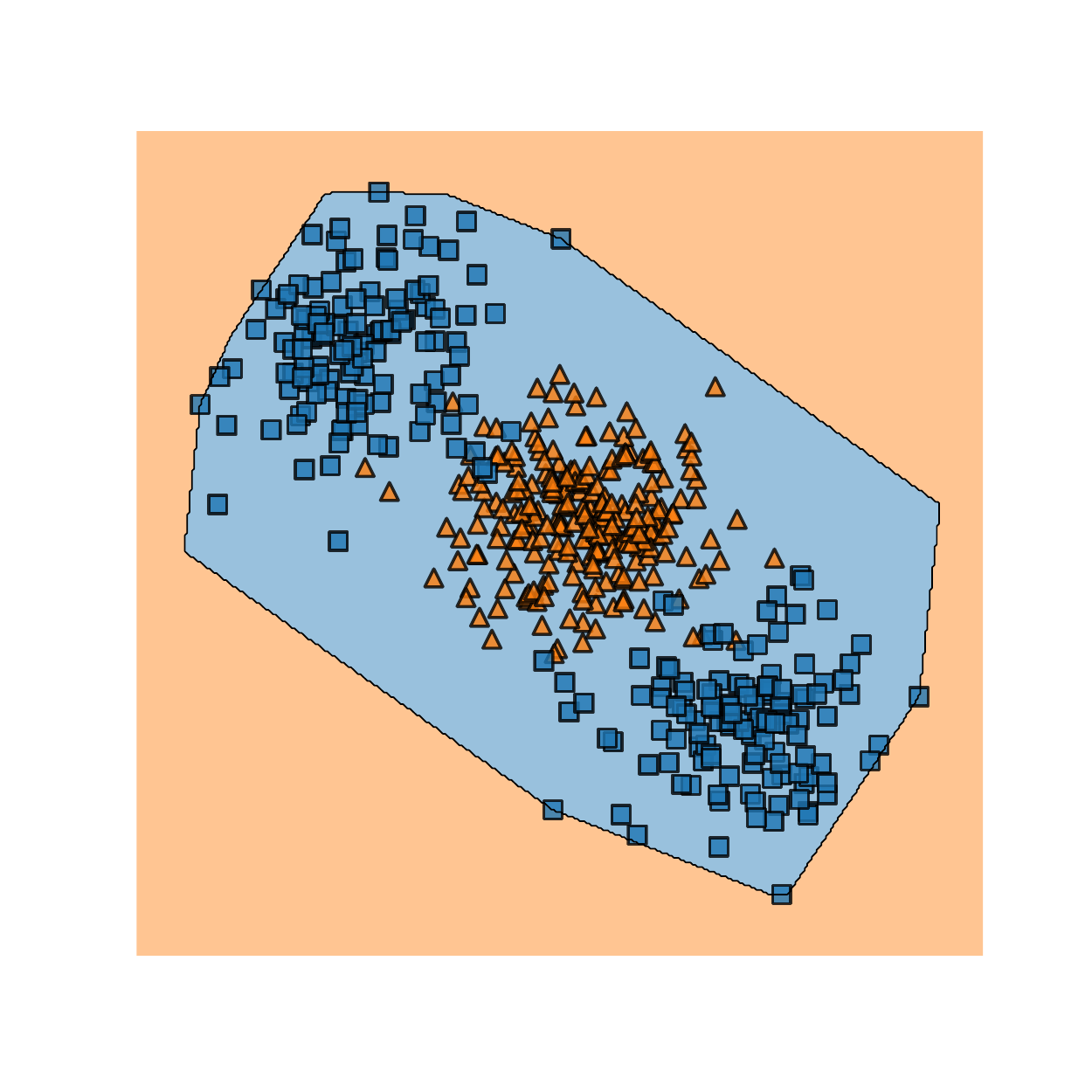}    & \includegraphics[width=0.24\columnwidth]{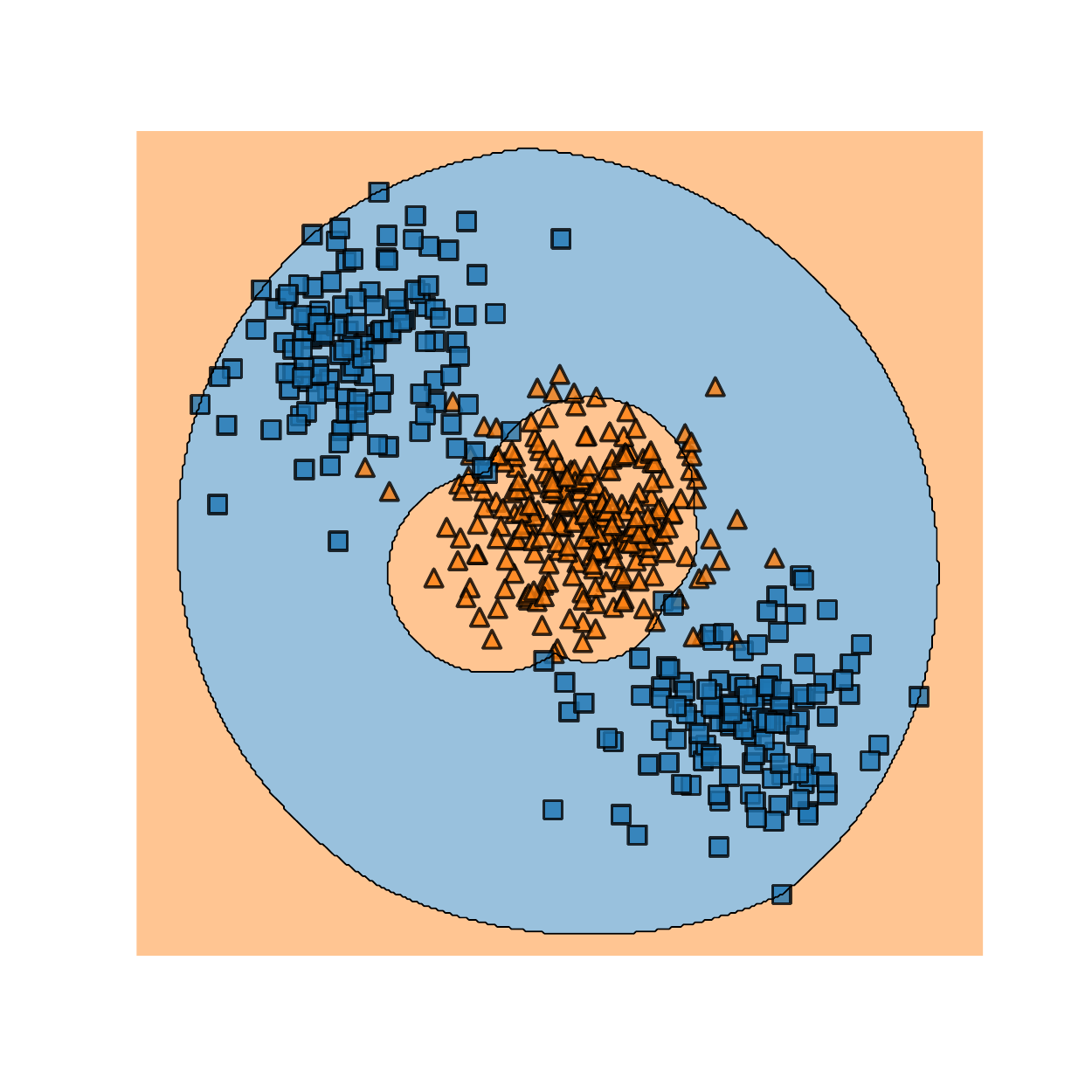} &
            \includegraphics[width=0.24\columnwidth]{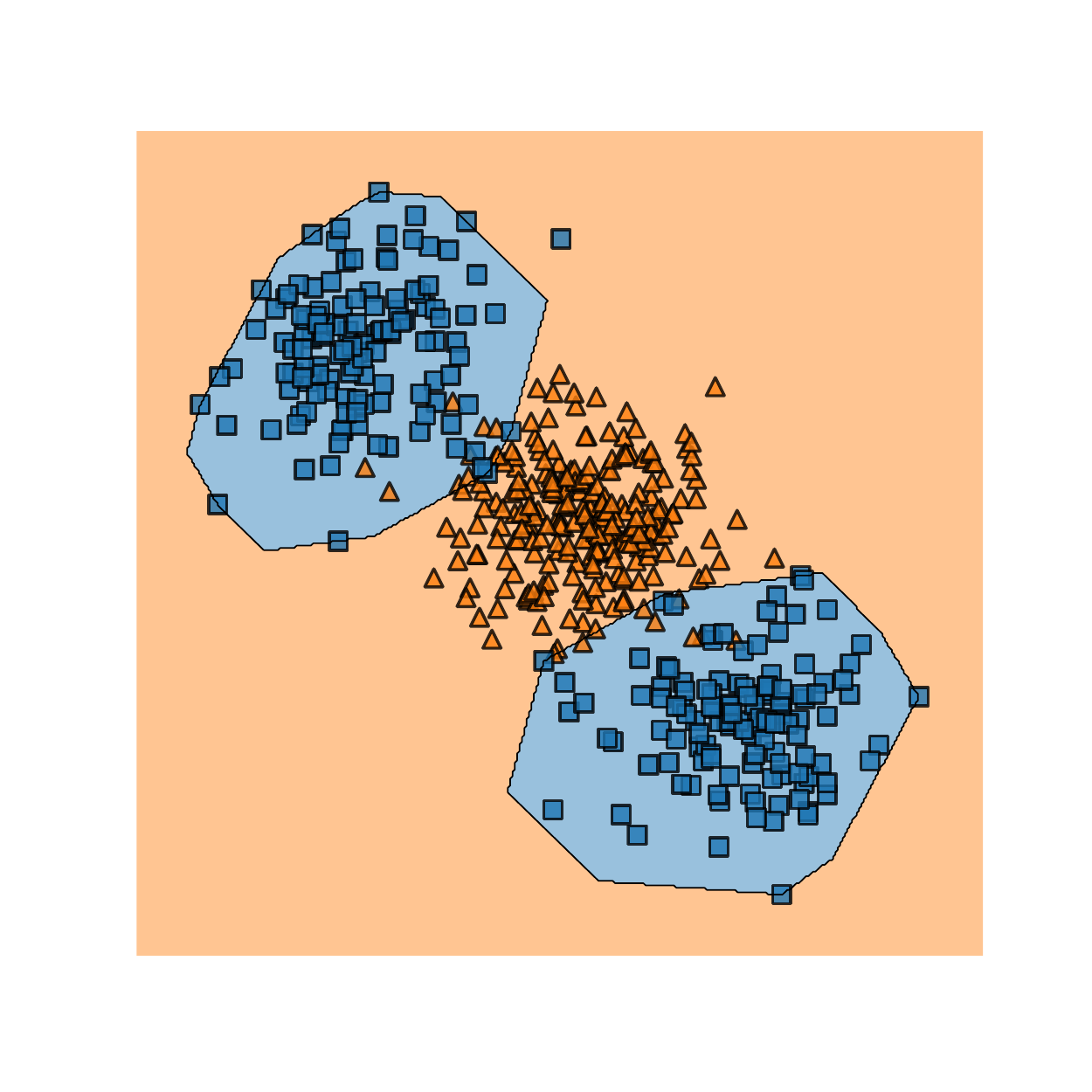}   &
            \includegraphics[width=0.24\columnwidth]{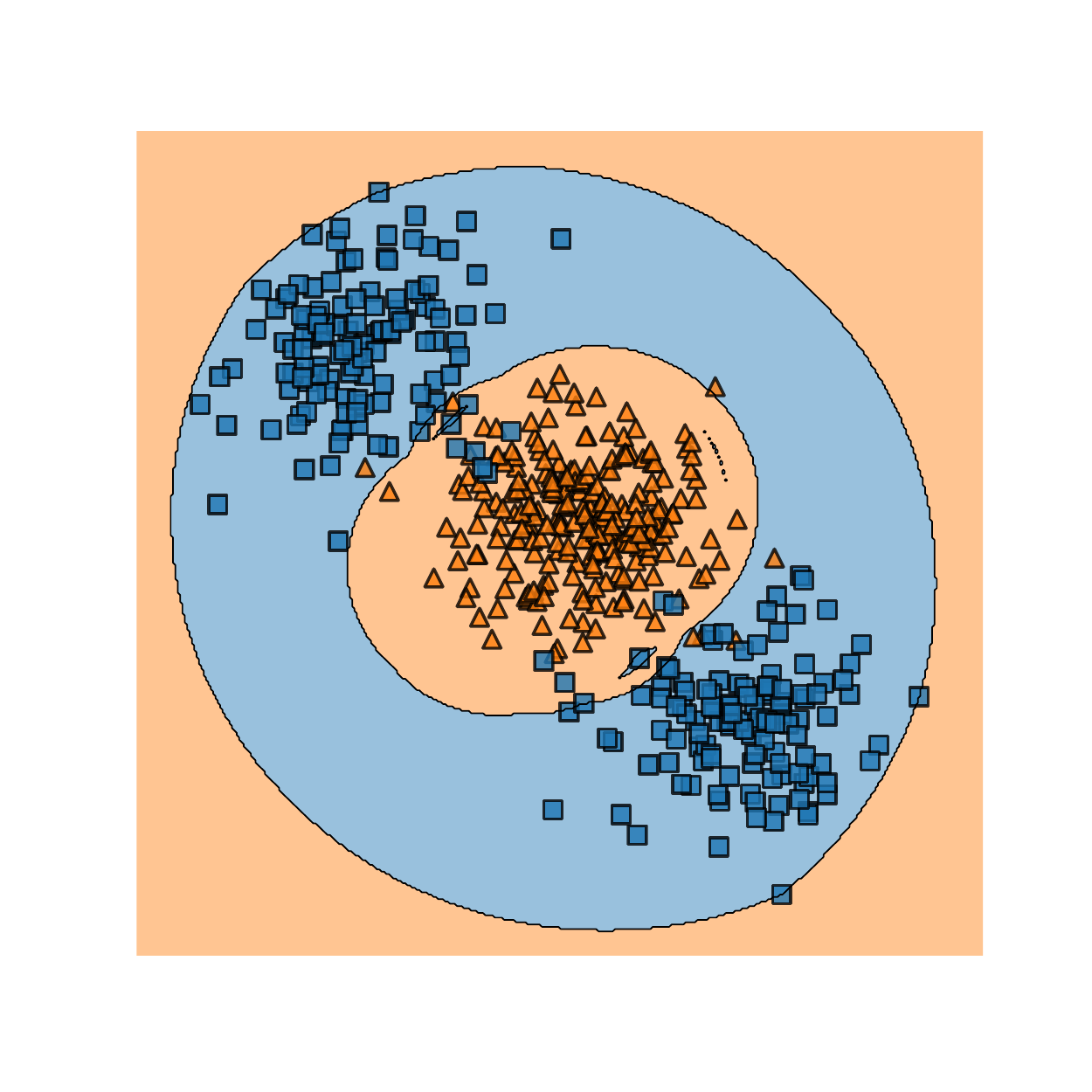}                                                                                                                                                                  \\
            \cmidrule(lr){1-4}
            \multicolumn{4}{c}{Moon Data}                                                                                                                                                                                                                                  \\
            ROC = 0.734                                                                                    & ROC = 0.783                                                                                  & ROC = 0.823                     & ROC = 0.911                  \\
            \includegraphics[width=0.24\columnwidth]{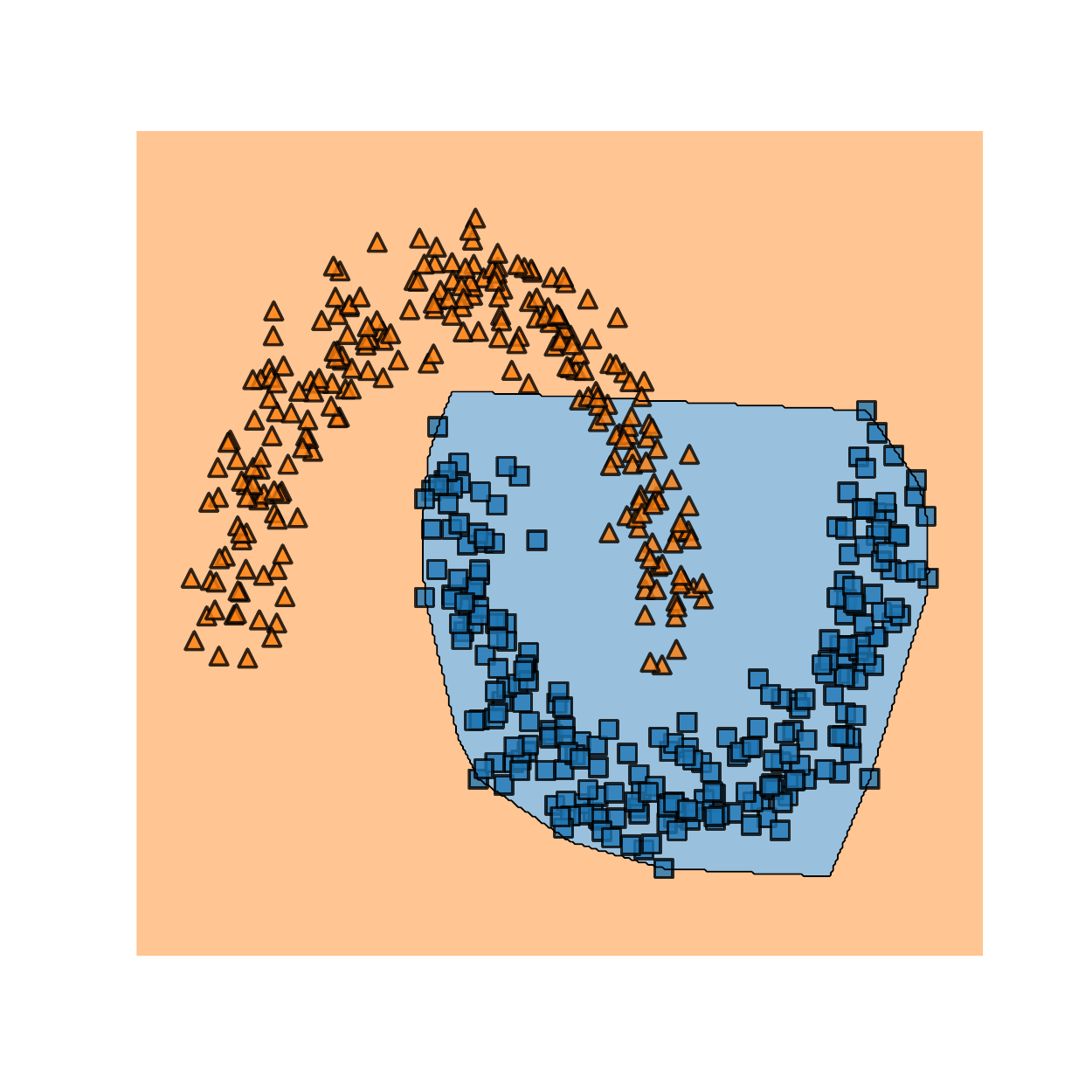}  &
            \includegraphics[width=0.24\columnwidth]{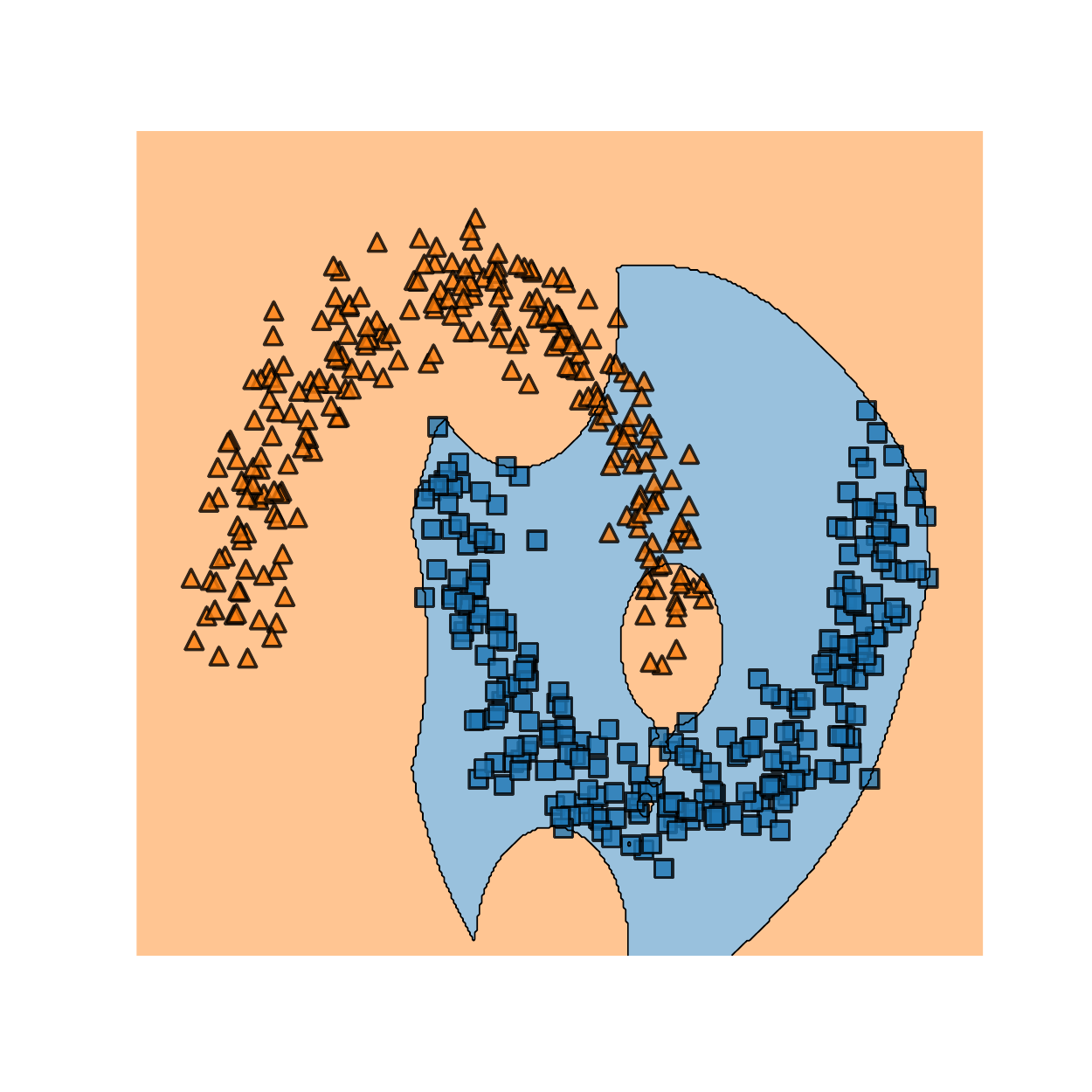} &
            \includegraphics[width=0.24\columnwidth]{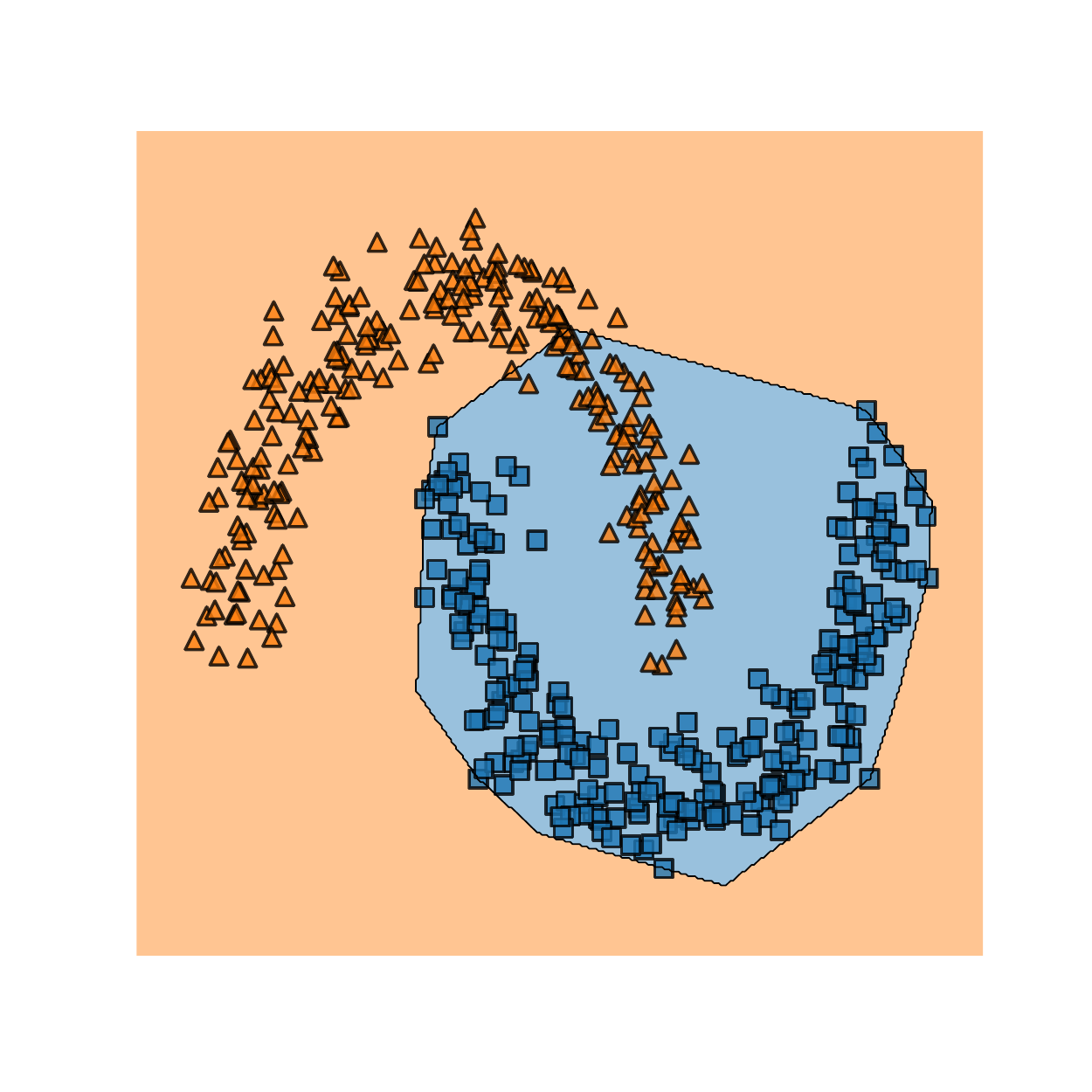} &
            \includegraphics[width=0.24\columnwidth]{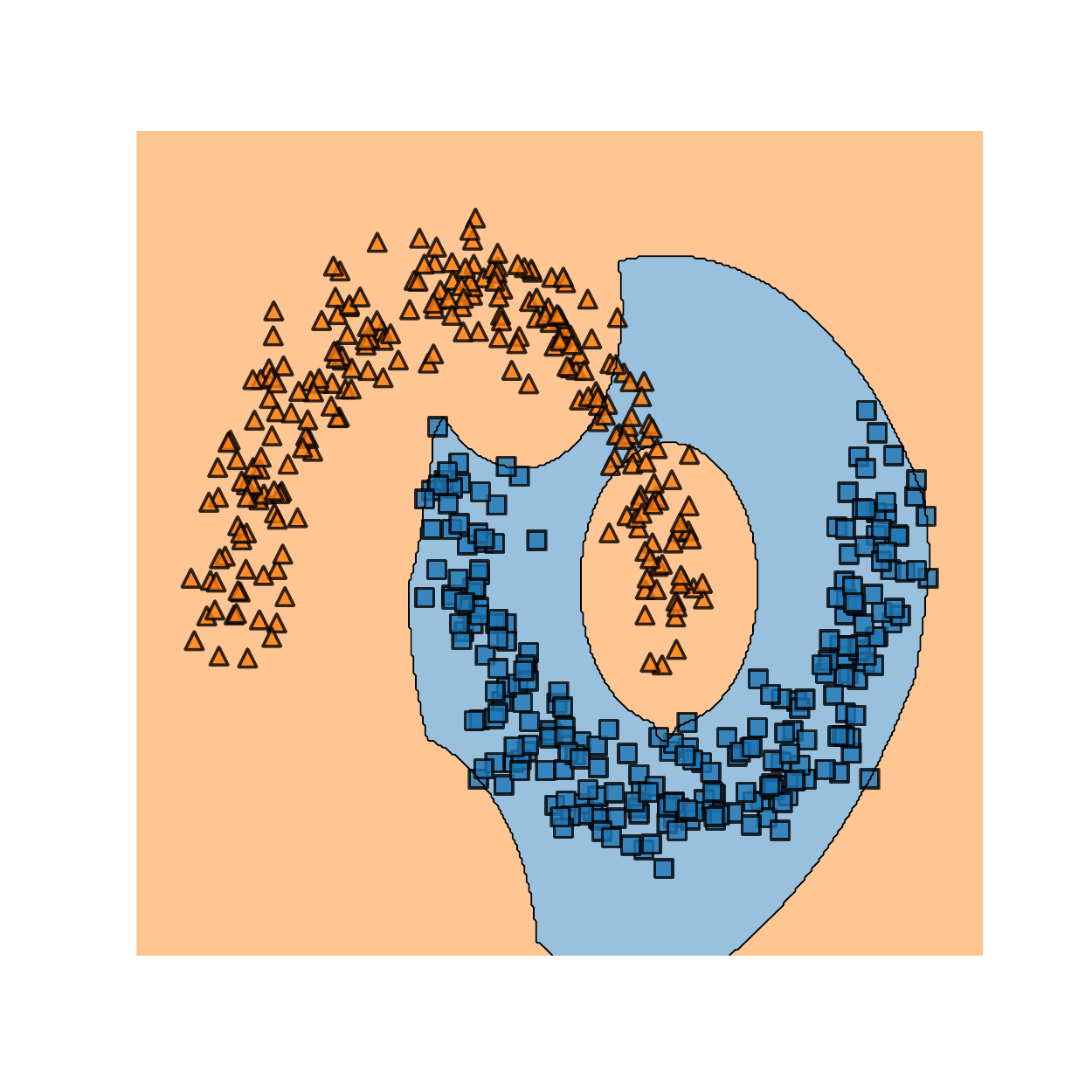}                                                                                                                                                                \\
            (a) $\varepsilon=1, K=linear$                                                                  & (b) $\varepsilon=1, K=rbf$                                                                   & (c) $\varepsilon=0.1, K=linear$ & (d) $\varepsilon=0.1, K=rbf$ \\
            \bottomrule
        \end{tabular}
    }
    \caption{Classification boundaries for different variants of \frocc. Blue-colored points are the ``normal'' class.}% In Section~\ref{sec:expt} we show results across a variety of datasets.}% and $\efrocc$ are dependent on the nature of the underlying data, $\kefrocc$ consistently performs better, independent of the nature of the data distribution (see Section~\ref{sec:expt} for more details).}
    \label{fig:heatmaps_boundaries}
\end{figure}

To understand the power of our method consider the one-class decision boundaries for various methods on two toy datasets visualized in Figure~\ref{fig:heatmaps_boundaries}. We see that \frocc with \(\varepsilon = 0.1\) computes a very desirable decision boundary for Multi-modal Gaussian data, a case which is known to be difficult for SVM to handle. In this case incorporating a standard Gaussian RBF kernel doesn't improve the ROC for \frocc. In the case of the Moon-shaped data, however, this kernel is able to direct \frocc to better decision boundaries which reflect in better ROC scores.

Through a range of experiments we show that the in most cases \frocc is able to achieve high scores in ROC as well as Precision@n measures with 2--3 orders of improvement in speed.

\paragraph{Our contributions.}
\begin{itemize}
    \item We define a powerful and efficient random-projection based one class classifier: \frocc.
    \item We show theoretically that \frocc is a stable classification algorithm.
    \item We perform an extensive experimental evaluation on several data sets that shows that our methods outperform traditional methods in both quality of results and speed, and that our methods are comparable in quality to recent neural methods while outperforming them in terms of speed by three orders of magnitude.
\end{itemize}

\paragraph{Organization.} We discuss related literature in Sec.~\ref{sec:related}. Formal definitions of our methods are presented in Sec.~\ref{sec:methods}. In Section~\ref{sec:theory} we provide analysis of stability of our methods. The experiments setup and results of an extensive experimental comparison with several baselines are discussed in Sec.~\ref{sec:expt}. We make some concluding remarks in Sec.~\ref{sec:conclusion}.

\section{Related Work}
\label{sec:related}

\paragraph{One Class Classification.}%
\label{sec:one-class-class}
One class classification has been studied extensively so we don't attempt to survey the literature in full detail, referring the reader instead to the survey of traditional methods by~\cite{khan2009survey}. A more recent development is the application of deep learning-based methods to this problem. A partial survey of these methods is available in~\cite{ruff2018deep}.

Unlike deep learning-based methods our work takes a more traditional geometric view of the one class problem in a manner similar to~\cite{Tax2004} and~\cite{Scholkopf2000} but we do not rely on optimization as a subroutine like their methods do. This gives us an efficiency benefit over these methods in addition to the improvement we get in terms of quality of results. An additional efficiency advantage our \frocc methods have is that we do not need to preprocess to re-center the data like these two methods do.

Our methods operate in the paradigm where the training data consists of a fully labeled set of examples {\em only} from the normal class. We do not use even a small number of examples from the negative class as done in, e.g., \cite{Borisyak-arxiv:2019}. Neither does our method use both labeled and unlabelled data as in the PAC model~\cite{valiant-cacm:1984}. Our methods are randomized in nature which put them in the same category as probabilistic methods that attempts to fit a distribution on the positive data. One examples of such a method is learning using Parzen window density estimation~\cite{cohen2008novelty}. We consider one such method as a baseline: \cite{fauconnier2009outliers} apply Minimum Co-variance Determinant Estimator~\cite{Rousseeuw_Driessen_1999} to find the region whose co-variance matrix has the lowest determinant, effectively creating an elliptical envelope around the positive data.

We also compare our algorithms to deep learning-based methods. Specifically we compare our methods to the Auto-encoder-based one class classification method proposed by~\cite{Sakurada_Yairi_2014}.

Another baseline we compare against is the deep learning extension of Support Vector Data Description (SVDD) proposed by~\cite{ruff2018deep}, where a neural network learns a transformation from input space to a useful space. In both cases our results are comparable and often better in terms of ROC while achieving a high speedup.

\paragraph{Random Projection-based Algorithms.}
Random projection methods have been productive in nearest-neighbor contexts due to the fact that it is possible to preserve distances between a set of points if the number of random dimensions chosen is large enough, the so-called Johnson-Lindenstrauss property, c.f., e.g. locality sensitive hashing~\cite{indyk-stoc:1998,ailon2006approximate}. Our methods do not look to preserve distances, only the boundary, and, besides, we are not necessarily looking for dimensionality reduction, which is key goal of the cited methods. A smaller number of projection dimensions may preserve distances between the training points but it {\em does not make the task of discriminating outliers any easier.} It may be possible that for the classifier to have good accuracy the dimensionality of the projection space be {\em higher} than the data space. As such our problem space is different from that of the space of problems that look for dimensionality reduction through random projection.~\cite{kleinberg-stoc:1997} uses projections to a collection of randomly chosen vectors to take a simpler approach than locality sensitive hashing but his problem of interest is also approximate nearest neighbor search. Some research has also focused on using random projections to learn mixtures of Gaussians, starting with the work of~\cite{dasgupta-focs:1999} but again this is a different problem space.

The closest work to ours in flavor is that of~\cite{vempala-jacm:2010} who attempts to find $k$-half-spaces whose intersection fully contains a given positively labelled data set. The key difference here is that our methods work with the intersections of hyper-rectangles that are unbounded in one dimension and not half-spaces. While the solution in~\cite{vempala-jacm:2010} can be a general convex body in the projection space, ours is specifically limited to cuboids. As such it is not possible to refine Vempala's method by taking a parameter $\varepsilon$ and subdividing  intervals (c.f. Section~\ref{sec:methods}), nor is it clear how it could be regularized for use in real-world applications.
\section{Fast Random-Projection based OCC}
\label{sec:methods}

\begin{definition}[\frocc($S,m,\varepsilon, K$)]
    \label{def:eps-rpocc}
    Assume that we are given a set of training points $S = \{\bm{x}_1, \ldots, \bm{x}_n\} \subseteq \RR^d$. Then, given an integer parameter $m>0$, a real parameter $\varepsilon \in (0, 1]$ and a kernel function $K(\cdot, \cdot)$, the {\em $\varepsilon$-separated \underline{F}ast \underline{R}andom-projection based \underline{O}ne-\underline{C}lass \underline{C}lassifier}  \frocc$(S,m,\varepsilon, K)$, comprises, for each $i$, $1 \leq i \leq m$,
    \begin{enumerate}
        \item A {\em classifying direction} $\bm{w}_i$ that is a unit vector chosen uniformly at random from $\bm{1}_d$, the set of all unit vectors of $\RR^d$, independent of all other classifying directions, and
        \item a set of intervals $R_i$ defined as follows: Let $S'_i = \{K(\bm{w}_i,\bm{x}_j) : 1 \leq j \leq n\}$ and $S_i = \{y_1, \ldots, y_n\}$ is a shifted and scaled version of $S'_i$ such that $y_i \in (0, 1), 1 \leq i \leq n$. Assume $y_1 \leq \cdots \leq y_n$. Then each interval of $R_i$ has the property that it is of the form $[y_j,y_{j+k}]$ for some $j \geq 1, k\geq0$ such that
              \begin{enumerate}
                  \item for all $t$ such that $0 \leq t < k$, $y_{j+t+1} - y_{j+t} < \varepsilon$,
                  \item $y_j - y_{j-1} > \varepsilon$ whenever $j-1 > 0$, and
                  \item $y_{j+k+1} - y_{j+k} > \varepsilon$ whenever $j+k+1 \leq n$.
              \end{enumerate}
    \end{enumerate}

    Given a query point $\bm{y} \in \RR^d$, \frocc$(S,m,\varepsilon, K)$ returns YES if for every $i$, $1 \leq i \leq m$, $K(\bm{w}_i,\bm{y})$ lies within some interval of $R_i$.
\end{definition}

In the simplest setting the kernel function $K(\cdot, \cdot)$ will be just the usual dot product associated with $\RR^d$.

\begin{figure}[ht!]
    \centering
    \includegraphics[width=0.7\columnwidth]{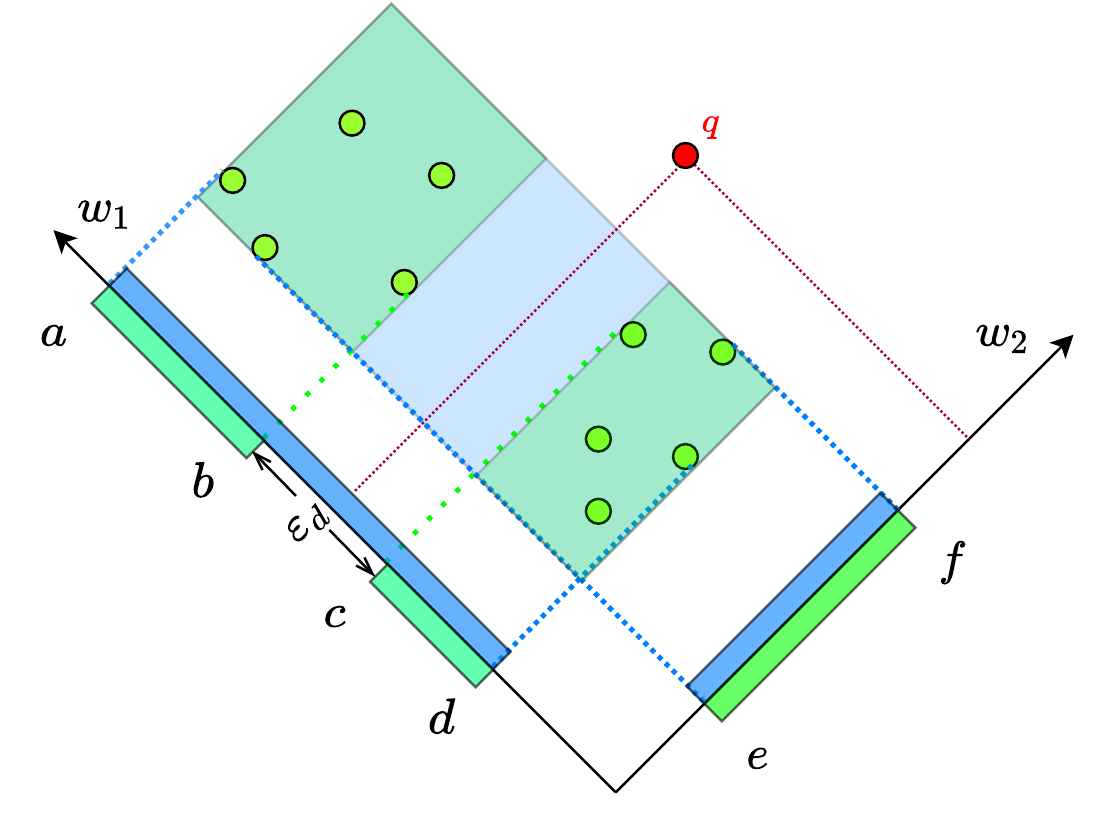}
    \caption{The test point $q$ is correctly classified as an outlier by $\bm{w_2}$ using \frocc since the projection lies outside the discriminating boundaries given by $e$ and $f$. However, $\bm{w_1}$ classifies $q$ incorrectly for $\varepsilon > \varepsilon_d$, since the discriminating boundary (blue region) is formed by the intervals $[a, d]$. If we chose  $\varepsilon \leq \varepsilon_d$, $\bm{w_1}$ correctly classifies $q$, since the discriminating boundaries (green regions) are given by the intervals $[a, b]$ and $[c, d]$.}
    \label{fig:frocc-example}
\end{figure}
We show a small example in Figure~\ref{fig:frocc-example}. The green points are training points. Two classifying directions $\bm{w}_1$ and $\bm{w}_2$ are displayed. The test point $\bm{q}$ has the property that $\langle \bm{q}, \bm{w}_1\rangle$ lies in $[a,d]$ but since $\langle \bm{q}, \bm{w}_2\rangle$ does not lie in $[e,f]$, $\bm{q}$ is said to be outside the class. Noting that the projections along $\bm{w}_1$ of the two clusters of the training set are well separated, a choice of $\varepsilon < \varepsilon_d$ leads to splitting of the interval $[a, d]$ into $[a, b]$ and $[c, d]$, thus correctly classifying $q$ to be an outlier.

In simple terms the intervals of $R_i$ have the property that the points of $S_i$ are densely scattered within each interval and there are wide gaps between the intervals that are empty of points of $S_i$. The density of the intervals is lower bounded by $1/\varepsilon$ and the width between successive intervals is lower bounded by $\varepsilon$. If we now look at Figure~\ref{fig:frocc-example} we note that since the interval $[b,c]$ has length more than $\varepsilon$, $R_1$ becomes $[a,b] \cup [c,d]$ and so the projection of $\bm{q}$ along $\bm{w}_1$ is also classified as being outside the class.

\paragraph{Computing \frocc.}
Algorithm~\ref{alg:frocc-train} outlines the training procedure of \frocc. The function $f$ in Line~\ref{line:f} is any valid kernel, which defaults to inner product \(\langle\cdot , \cdot\rangle\). Algorithm~\ref{alg:interval-tree} shows how the interval tree is created for each classification direction. Unlike many optimization based learning methods, such as neural networks, our algorithm needs a single pass through the data. %Note that if \(\varepsilon = 1\), the computation is greatly simplified and the resultant subroutine is presented in Algorithm~\ref{alg:minmax}.

\begin{algorithm}[ht!]
    \caption{\frocc Training}
    \label{alg:frocc-train}
    \SetAlgoLined
    \SetKwInOut{Input}{input}
    \SetKwInOut{Output}{output}

    \KwData{$S = \{\bm{x}_1, \ldots, \bm{x}_n\} \subseteq \RR^d$: training set}
    \Input{m, $\varepsilon$}
    \Output{$[R_i]_{i=1}^{m}$: list of m interval trees}
    \BlankLine
    Initialize $result\leftarrow$ \emph{empty list}\\
    \For{$i\leftarrow 1$ \KwTo $m$}{
        $\ell_i\leftarrow inf$\\
        $u_i\leftarrow 0$\\
        $\bm{w}_i\thicksim Uniform(\bm{1}_d)$ \tcp{sample a d-dimensional unit vector uniformly}\
        Initialize projectionList $\leftarrow$ \emph{empty list} \\
        \For{$j\leftarrow 1$ \KwTo $n$}{
            $projection\leftarrow f(x_j, w_i)$\label{line:f}\\
            projectionList.append(projection)\\
        }
        sort(projectionList)\\
        $R_i\leftarrow$ \emph{IntervalTree(projectionList, $\varepsilon$)} \\
        result.append($R_i$)
    }
\end{algorithm}

\begin{algorithm}[ht!]
    \caption{IntervalTree subroutine for $\varepsilon \in (0, 1)$}
    \label{alg:interval-tree}
    \SetAlgoLined
    \SetKwInOut{Input}{input}
    \SetKwInOut{Output}{output}

    \Input{projectionList, $\varepsilon$}
    \Output{R: interval tree}
    Initialize $R\leftarrow$ \emph{empty interval tree}\\
    Initialize $start\leftarrow projectionList[1]$, $end\leftarrow projectionList[1]$ \\
    margin = $(\max(projectionList) - \min(projectionList)) \cdot \varepsilon$ \\
    \For{$i\leftarrow 2$ \KwTo $n$}{
        \eIf{$projectionList[i] < end + margin$}{
            $end\leftarrow projectionList[i]$ \\
        }{
            \tcp{point is outside margin}
            R.insert([start, end]) \\
            $start\leftarrow projectionList[i]$ \\
            $end\leftarrow projectionList[i]$ \\
        }
    }
    \If{$start \neq end$}{
        \tcp{handle last interval}
        R.insert([start, end])
    }
\end{algorithm}

% \begin{algorithm}[htb]
%     \caption{IntervalTree subroutine for $\varepsilon = 1$}
%     \label{alg:minmax}
%     \SetAlgoLined
%     \SetKwInOut{Input}{input}
%     \SetKwInOut{Output}{output}

%     \Input{projectionList}
%     \Output{R: interval list}
%     $\ell\leftarrow \min(projectionList)$\\
%     $u\leftarrow \max(projectionList)$\\
%     R.insert($[\ell, u]$)
% \end{algorithm}

The standard method for choosing uniformly random vectors in $\RR^d$ is to sample from a spherical $d$-dimensional Gaussian centered at the origin and then normalizing it so that its length is 1. Folklore and a quick calculation suggests that independently picking $d$ random variables with distribution $N(0,1)$ is good enough for this purpose. The spherical symmetry of the Gaussian distribution ensures that the random vector is uniformly picked from all possible directions.

\paragraph{Time complexity.} The time taken to compute the boundaries of the $\frocc$ is $O(mnd)$ for $\varepsilon=1$ (assuming the kernel is dot product, which takes $O(d)$). For $\varepsilon < 1$ additional factor of $\theta(mn \log n)$ is required to create $\varepsilon$-separated intervals along each classifying direction.

\section{Theoretical Analysis}
\label{sec:theory}
In this section we state and prove certain properties of \frocc which states that \frocc is a stable classifier whose performance can be reliably tweaked by modifying the number of classification dimensions \(m\) and separation parameter \(\varepsilon\). In particular, we show that:
% \begin{inparaenum}
%     \item Bias of \frocc decreases exponentially with the size of the training set,
%     \item Precision and model size of \frocc varies monotonically with \(\varepsilon\) and \(m\), and
%     \item \frocc's error decay exponentially with \(m\)
% \end{inparaenum}
\begin{enumerate}
    \item Bias of \frocc decreases exponentially with the size of the training set if the training set is drawn from a spatially divisible distribution,
    \item Precision and model size of \frocc increase monotonically with \(m\) and decrease monotonically with \(\varepsilon\), and
    \item Probability of \frocc making an error decays exponentially with \(m\).
\end{enumerate}
\ifcutlevone
    % \subsection{Proof of stability of \frocc}
    \subsection{Stability of \frocc}
    \label{sec:stability}

    Following the analysis of~\cite{bousquet-jmlr:2002} we prove that \frocc generalizes in a certain sense. Specifically we will show that under some mild assumptions the bias of \frocc with $\varepsilon =1$ goes to 0 at an exponential rate as the size of the training set increases. We first show a general theorem which should be viewed as a variant of Theorem 17 of~\cite{bousquet-jmlr:2002} which is more suited to our situation since we address a one-class situation and our algorithm is randomized. Our Theorem~\ref{thm:rand-bousquet} is of general interest and may have applications beyond the analysis of \frocc.

    \subsubsection{Terminology and notation}
    Let $\mathcal{S}$ denote the set of all finite labelled point sets of $\RR^d$ where each point is labelled either 0 or 1. We assume we have a training set $S\in \mathcal{S}$ chosen randomly as follows: We are given an unknown distribution $D$ over $\RR^d \times \{0,1\}$ and $S$ comprises $k$ samples with label 1 drawn i.i.d. from $D$. For a point $z = (x,y) \in \RR^d \times \{0,1\}$ we use $\data(z)$ to denote $x$ and $\lab(z)$ to denote $y$.

    Let $\mathcal{R}$ be the set of all finite strings on some finite alphabet, and let us call the elements of $\mathcal{R}$ {\em decision strings}. Further, let $\mathcal{F}$ be the set of all {\em classification functions} from $\RR^d$ to $\{0,1\}$. Then we say a classification map $\Phi: \mathcal{S} \times \mathcal{R} \rightarrow \mathcal{F}$ maps a training set and a decision string to a classification algorithm

    With this setup we say that a {\em randomized classification algorithm} $A$ takes a training $S\in \mathcal{S}$ as input, picks a random decision string $r \in \mathcal{R}$ and returns the classification function $\Phi(S,r)$ which we denote $A(S,r)$. $A(S,r)$ is a randomly chosen element of $\mathcal{F}$ which we call a classifier. Given an $r$, $A(S,r)$ is fixed but we will use $A_S$ to denote the randomized classifier which has been given its training set $S$ but is yet to pick a random decision string. Now, given a $z\in \RR^d$, the {\em loss} function of a randomized classifier $A_S$ is given by
    \[V(A_S,z) = \exsub{r}{\ind{A(S,r)(\data(z)) \ne \lab(z)}},\]
    where the expectation is over the randomness of the decision string $r$. In other words the loss is equal to $ \prsub{r}{A(S,r)(\data(z)) \ne \lab(z)}$.

    We define the {\em risk} of $A_S$ as
    \[R(A_S) = \exsub{S,z}{V(A_S,z)},\]
    where the expectation is over the random choice of $S$ and of the point $z$. Note that $V(\cdot,\cdot)$ already contains an expectation on the randomness associated with the decision string which is inherent in $A_S$

    The {\em empirical risk} of $A_S$ is defined
    \[R_e(A_S) = \frac{1}{|S|} \sum_{z \in S} V(A_S,z).\]
    We introduce notation for two modifications of the training set $S$. For $1 \leq i \leq |S|$ we say that $S^{\setminus i} = S \setminus \{z_i\}$ and $S^i = S^{\setminus i} \cup \{z_i'\}$ where $z_i'$ is chosen from $D$ independently of all previous choices. We now define a notion of stability.

\fi

The following notion of stability is a modification of the definition of stability for classification algorithms given by~\cite{bousquet-jmlr:2002}.
\begin{definition}[0-1 uniform classification stability]
    Suppose we have a randomized classification algorithm $A$ and a loss function $V(\cdot,\cdot)$ whose co-domain is $[0,1]$. Then $A$ is said to have {\em 0-1 uniform classification stability} $(\beta,\eta)$ if for all $S \subseteq \RR^d$, for every $i$ such that $1 \leq i \leq |S|$, and for every $z \in \RR^d \times \{0,1\}$
    \[|V(A_S,z) - V(A_{S^{\setminus i}},z)| \leq \beta \]
    with probability at least $1 - \eta$.

    If $\beta$ is $O(1/|S|)$ and $\eta$ is $O(e^{-c|S|})$ for some constant $c > 0$, we say that $A$ is {\em 0-1 uniform classification stable.}
\end{definition}
We prove that a 0-1 uniform classification stable algorithm has low bias and converges exponentially fast in the size of $S$ to its expected behavior. This is a general result that may be applicable to a large class of randomized classification algorithm. We also prove that FROCC is 0-1 uniform classification stable under mild conditions on the unknown distribution.

\ifcutlevtwo
    We also introduce a condition that the unknown distribution $D$ has to satisfy in order for us to show the desired stability property.
    \begin{definition}[Spatial divisibility]
        Suppose that $\mu: \RR^d \rightarrow [0,1]$ is a probability density function. For any set $T$ containing $d$ points $x_i,\ldots, x_d \in \RR^d$ let $H_{T+}$ and $H_{T-}$ be the two half-spaces defined by the hyper-plane containing all the points of $T$. We say that $\mu$ is {\em spatially divisible} if for any set $d+1$ randomly chosen points $X = \{X_1,\ldots,X_d\}$ and $Y$ chosen independently according to density $\mu$ and any $A,B\subset \RR^d$ such that $\mu(A), \mu(B) > 0$,
        \begin{align*}
            \prob{Y \in H_{X+} | X_i \in A, Y \in B} & >0, \mathrm{and~} \\
            \prob{Y \in H_{X-} | X_i \in A, Y \in B} & >0.
        \end{align*}
    \end{definition}
    Note that the set $X$ defines a hyper-plane in $d$ dimensions with probability 1 for any distribution defined on a non-empty volume of $\RR^d$. Most standard distributions have the spatial divisibility property. For example, it is easy to see that a multi-variate Gaussian distribution has this property. If we pick a point uniformly at random from within a convex $d$-polytope then too the spatial divisibility property is satisfied. To see this we note that the only way it could not be satisfied is if the points $X_i$ and $Y$ are picked from the surface of the $d$-polytope which is an event of probability 0.

    \subsubsection{Results}
    With the definition of stability given above we can prove the following theorem which shows that the deviation of the empirical risk from the expected risk tends to 0 exponentially as $|S|$ increases for a class of randomized algorithms that includes \frocc.
    \begin{theorem}
        \label{thm:rand-bousquet}
        Suppose we have a randomized classification algorithm $A$ which has 0-1 uniform classification stability $(\beta, \eta)$ with $0 < \beta, \eta < 1$, and suppose this algorithm is trained on a set $S$ drawn i.i.d. from a hidden distribution in such a way that the random choices made by $A$ are independent of $S$, then, for any $\varepsilon > 0$,
        \begin{equation}
            \label{eq:bousquet}
            \prob{|R(A,S) - R_e(A,S)| > \varepsilon} \leq 2 \exp \left\{- \frac{2\varepsilon^2}{|S|\beta^2}\right\} + 2|S|\eta.
        \end{equation}
        Moreover if $A$ is 0-1 uniform classification stable then the RHS of \eqref{eq:bousquet} tends to 0 at a rate exponential in $|S|$.
    \end{theorem}

    \begin{proof} %[Proof of Theorem~\ref{thm:rand-bousquet}]
        As in~\cite{bousquet-jmlr:2002} we will use McDiarmid's inequality~\cite{mcdiarmid-sc:1989} to bound the deviation of the empirical risk from the risk. For completeness we state the result:
        \begin{lemma}[\cite{mcdiarmid-sc:1989}]
            \label{lem:mcdiarmid}
            Given a set of $m$ independent random variables $Y_i, 1 \leq i \leq m$ and a function $f$ such that $|f(x_1, \ldots, x_m) - f(y_1, \ldots, y_m)| < c_i$ whenever $x_j = y_j, 1 \leq j \leq m, j \ne i$, and $x_i \ne y_i$, then
            \[
                \prob{|\ex{f(Y_1, \ldots, Y_m)} - f(Y_1, \ldots, Y_m)|  > \varepsilon} \leq e^{\left(-\frac{2\varepsilon^2}{\sum_{i=1}^m c_i^2}\right)}.\]
        \end{lemma}
        To use Lemma~\ref{lem:mcdiarmid}  we have to show that $R_e(A,S)$ is a $2\beta$-Lipschitz function. However, this property is only true with a certain probability in our case. In particular we will show that for all $i$, $1 \leq i \leq |S|$, $|R_e(A,S) - R_e(A,S^i)|\leq 2\beta$ with probability at least $1 - 2|S|\eta$. To see this we note that for any $i$ and any $z_j \in S$ (including $j = i$).
        \begin{align*}
            |V(A_S,z_j) - V(A_{S^i},z_j)| = & |V(A_S,z_j) - V(A_{S^{\setminus i}},z_j)         \\
                                            & +  V(A_{S^{\setminus i}},z_j) - V(A_{S^i},z_j)|.
        \end{align*}
        Using the definition of $(\beta,\eta)$ 0-1 uniform classification stability and the triangle inequality, we get that the RHS is bounded by $2\beta$ with probability at least $1 - 2\eta$. If $A_j$ is the event that the RHS above is bounded by $2\beta$ then we want to bound the probability of the event that this is true for all $j$, $1 \leq j \leq |S|$, i.e., the event $F = \cap_{j=1}^{|S|} A_j$. Since $\prob{\overline{A_j}} \leq 2\eta$, we can say that $\prob{F}$ is at least $1 - 2|S|\eta$.

        Now, let $E$ be the event that $|R(A,S) - R_e(A,S)| > \varepsilon$ and $F$ defined above be the event that $R_e(A,S)$, which is a function of the vector of $|S|$ elements chosen independently to form $S$ is $2\beta$-Lipschitz. We know that
        \[\prob{E} = \prob{E\cap F} + \prob{E \cap \overline{F}},\]
        and so we can say that
        \begin{equation}
            \label{eq:conditioning}
            \prob{E} \leq  \prob{E | F}\prob{F} + \prob{\overline{F}}.
        \end{equation}
        We have already argued that the second term on the RHS is upper bounded by $2 |S|\eta$ so we turn to the first term. To apply Lemma~\ref{lem:mcdiarmid} to the first term we note that $E$ depends on the random selection of $S$ which are selected independently of each other and, by assumption, independent of the random choices made by $A$. Therefore the random collection $\{ V(A_S,z) : z \in S\}$ is independent even when conditioned on $F$ which is determined purely by the random choices of $A$ and is true for {\em every} choice of set $S$. Once this is noted then we can use McDiarmid's inequality to bound the first term of Equation~\eqref{eq:conditioning}. We ignore $\prob{F}$ by upper bounding it by 1.
    \end{proof}

    We now show that Theorem~\ref{thm:rand-bousquet} applies to \frocc with $\varepsilon = 1$ if the unknown distribution $D$ is spatially divisible.
    \begin{proposition}
        \label{prp:rpocc}
        If the unknown distribution $D$ is spatially divisible then \frocc with $\varepsilon = 1$ is 0-1 uniform classification stable. %and, consequently,
    \end{proposition}

    \begin{proof} %[Proof of Proposition~\ref{prp:rpocc}]
        For any $i$, $1 \leq i \leq n$, where $n = |S|$,  we note that when $\varepsilon =1$ the entire interval spanned by the projections in any direction is said to contain inliers, and so if $x_i$ lies strictly within the convex hull of $S$ then the removal of $i$ from $S$ does not affect the classifier. This is because the boundaries in any direction are determined by the points that are part of the convex hull. Let us denote this set $\mbox{conv}(S)$. Therefore, since $V(\cdot,\cdot)$ takes value at most 1, we deduce the following upper bound
        \[|V(A_S,z) - V(A_{S^{\setminus i}},z)|\leq \prob{x_i \in \mbox{conv}(S)}.\]
        To bound the probability that a point $x_i$ lies in the convex hull of $S$ we use an argument that Efron~\citep{efron-biometrika:1965} attributes to R\'enyi and Sulanke~\citep{renyi-z-wahr:1963,renyi-z-wahr:1964}. Given a region $B \subset \RR^d$ let us denote by $D(B)$ the probability that a point drawn from the unknown distribution $D$ places a point in $B$. Now given any $d$ points, say $x_1, \ldots, x_d \in S$, we divide $\RR^d$ into two regions $A_{X+}$ that lies on one side of the hyper-plane defined by $x_1, \ldots, x_d$ and $A_{X-}$ that lies on the other side. Then the probability that $x_1, \ldots, x_d$ all lie in $\mbox{conv}(S)$ is equal to the probability that all the remaining points of $S$ are either in $A_{X+}$ or in $A_{X-}$, i.e.,
        \[\prob{x_1, \ldots, x_d \in \mbox{conv}(S)} = D(A_{X+})^{(n-d)} + D(A_{X-})^{(n-d)}.\]
        From this, using the union bound we can say that%
        \[\prob{x_i \in \mbox{conv}(S)} \leq \sum_{Y \in \mathcal{P}(S\setminus \{x_i\}, d)}D(A_{Y+})^{(n-d)} + D(A_{Y-})^{(n-d)},\]%
        where $\mathcal{P}(A, d)$ is the set of all subsets of $A$ of size $d$. This can further be bounded as%
        \[ LHS \leq 2\binom{n-1}{d} \max_{Y \in \mathcal{P}(S\setminus \{x_i\}, d)} \max\{D(A_{Y+})^{(n-d)}, D(A_{Y-})^{(n-d)}\}.\]%
        Since we have assumed that $D(A_{Y\pm})<1$ for all $Y\subset S$ such that $|Y|=d$ and $x_i \notin Y$, using the fact that $\binom{n}{k} \leq (en/k)^k$, we can say that there is an $\alpha$ such that $0<\alpha < 1$ and a constant $C > 0$ such that
        \[|V(A_S,z) - V(A_{S^{\setminus i}},z)|\leq C\cdot (n-1)^d \alpha^n\]
        almost surely. Since the RHS goes to 0 exponentially fast in $n$ with probability 1, we can say that \frocc is 0-1 uniform classification stable when $\varepsilon =1$.
    \end{proof}

    From Theorem~\ref{thm:rand-bousquet} and Proposition~\ref{prp:rpocc} we get
    \begin{theorem}
        \label{thm:rpocc-bias}
        For a training set $S$ chosen i.i.d. from a spatially divisible unknown distribution, if $\varepsilon$ is set to 1 then $R_e(\frocc,S)$ converges in probability to $R(\frocc, S)$ exponentially fast in $|S|$.
    \end{theorem}

    We feel that it should be possible to use the framework provided by Theorem~\ref{thm:rand-bousquet} to show that \frocc is stable even when $\varepsilon$ lies in $(0,1)$, however this problem remains open.
\fi

\subsection{Monotonicity properties of \frocc}
We now show that the precision and the model size of \frocc is monotonic in the two parameters $m$ and $\varepsilon$ but in opposite directions. The precision and model size increase as $m$ increases and decrease as $\varepsilon$ increases. This result is formalized as Proposition~\ref{prp:varying} below and follows by a coupling argument.

We first introduce some terms. On the $\sigma$-algebra induced by the Borel sets of $\cup_{i\geq 1} \RR^i$ we define the probability measure $\Psub{m,\varepsilon}{\cdot}$ induced by \frocc with parameters $\varepsilon, m > 0$. Note that $\varepsilon, m > 0$, \frocc is a random collection of $m$-dimensional cubes in $\RR^m$. These cubes are induced by taking the products of intervals in each of the $m$ dimensions. Each dimension has between 1 and $|S|$ intervals when trained on $S$. We define the {\em model size} of \frocc to be the sum of the number of intervals, i.e., if there are $k_i$ intervals along dimension $i$, $1 \leq i \leq m$, the model size is $\sum_{i=1}^m 2k_i$. We will use the letter $M_{m,\varepsilon}$ to denote the model size of \frocc with parameters $m, \varepsilon$. Note that this is a random variable.

\begin{proposition}
    \label{prp:varying}
    Given \frocc trained on a finite set $S \subset \RR^d$, for $\varepsilon,  \varepsilon' > 0$, $m,m' > 0$ and any $x \in \RR^d$ we have that
    \begin{enumerate}
        \item whenever $\varepsilon < \varepsilon'$
              \begin{enumerate}
                  \item $\Psub{m,\varepsilon}{x \mbox{ is classifed YES}} \leq \Psub{m,\varepsilon'}{x \mbox{ is classifed YES}}$,
                  \item $\E{m,\varepsilon}{M_{m,\varepsilon}} \geq \E{m,\varepsilon'}{M_{m,\varepsilon'}}$, and,
              \end{enumerate}
        \item whenever $m < m'$
              \begin{enumerate}
                  \item $\Psub{m,\varepsilon}{x \mbox{ is classifed YES}} \leq \Psub{m',\varepsilon}{x \mbox{ is classifed YES}}$,
                  \item $\E{m,\varepsilon}{M_{m,\varepsilon}} \geq \E{m',\varepsilon}{M_{m',\varepsilon}}$.
              \end{enumerate}
    \end{enumerate}
\end{proposition}
\begin{proof}
    For 1(a) and 1(b) we couple the probability distributions $\Psub{m,\varepsilon}{\cdot}$ and $\Psub{m,\varepsilon'}{\cdot}$ by choosing $m$ random unit vectors, $\vec{w}_1, \ldots, \vec{w}_m$, and then constructing one instance of \frocc with separation \(\varepsilon\) and one instance of \frocc with separation $\varepsilon'$ using the same $m$ random unit vectors.

    Now, consider some vector $\vec{w}_i$ and let the $x_j^i, 1 \leq i \leq |S|$ be the projections of the training points on $\vec{w}_i$. Let the intervals formed by $\varepsilon$ separated \frocc be $\{[a_{ij}, b_{ij}] : 1 \leq j \leq k_i\}$ and those of $\varepsilon'$ separated \frocc be $\{[a'_{ij}, b'_{ij}] : 1 \leq j \leq k'_i\}$. We know that $a'_{i1} = a_{i1} = \min x_j^i$. Clearly each $[a_{ij},b_{ij}]$ is contained within some $ [a'_{ij'},b'_{ij'}]$ since the distance between subsequent points within $[a_{ij},b_{ij}]$ is at most $\varepsilon$ which is strictly less than $\varepsilon'$. This proves 1(a) and also implies that $k_i' \leq k_i$ which proves 1(b).

    For 2(a) and 2(b) we couple the probability distributions $\Psub{m,\varepsilon}{\cdot}$ and $\Psub{m',\varepsilon}{\cdot}$ by choosing $m$ random unit vectors, $\vec{w}_1, \ldots, \vec{w}_m$, and then constructing two instances of \frocc. Then we choose an additional $m' - m$ vectors $\vec{w}_{m+1}, \ldots, \vec{w}_{m'}$ and compute intervals corresponding to these and add them to the second instance of \frocc. 2(b) follows immediately. For 2(a) we note that any $x \in \RR^d$ classified as YES by the second model must also be classified as YES by the first model since it is classified as YES along the first $m$ dimensions of the second model which are common to both models.
\end{proof}

From Proposition~\ref{prp:varying} it follows that expected model size and false positive probability vary inversely, i.e., the larger the model size the smaller the false positive probability and vice versa.

Proposition~\ref{prp:varying} is also consistent with the following upper bound that is easy to prove
\begin{fact}
    For \frocc constructed on training set $S = \{\vec{x}_1, \ldots, \vec{x}_n\}$,
    $$2m \leq M_{m,\varepsilon} \leq m \cdot \max\left\{2, \frac{\delta_{\max}}{\varepsilon}\right\}$$
    with probability 1,
    where $\delta_{\max} = \max_{1 \leq i,j \leq n} \|\vec{x}_i - \vec{x}_j\|_2$.
\end{fact}
The first inequality is trivial. The second follows by observing that along any dimension the gaps between intervals must be at least $\varepsilon$ and the length of the interval between the maximum and minimum projection is at mode $\delta_{\max}$.

\ifcutlevtwo
    \subsection{\frocc's error decays exponentially with \texorpdfstring{\(m\)}{m}}
    \begin{definition}
        Given a finite set of points $S = \{\bm{x}_1, \ldots, \bm{x}_n\} \subseteq \RR^d$ and any $\bm{w} \in \bm{1}_d$, we say that
        $\ell_{\bm{w}} = \min_{j=1}^n \langle \bm{x}_j, \bm{w}\rangle,$
        and
        $u_{\bm{w}} = \max_{j=1}^n \langle \bm{x}_j, \bm{w}\rangle$.
        Now, given a point $\bm{y} \subset \RR^d$ and $\bm{w} \in \bm{1}_d$, we say that the distance of $\bm{y}$ from $S$ {\em along the direction} $\bm{w}$ is defined as
        \[\alpha(\bm{y},S, \bm{w}) =
            \left\{ \begin{array}{ll}
                \ell_{\bm{w}} - \langle \bm{y}, \bm{w}\rangle & \mbox{if } \langle \bm{y}, \bm{w}\rangle < \ell_{\bm{w}}                     \\
                0                                             & \mbox{if } \ell_{\bm{w}} \leq  \langle \bm{y}, \bm{w}\rangle \leq u_{\bm{w}} \\
                \langle \bm{y}, \bm{w}\rangle - u_{\bm{w}}    & \mbox{o.w.}
            \end{array}
            \right.
        \]
        Further we say that $\bm{w} \in \bm{1}_d$ is a {\em discriminator} of $\bm{y}$ from $S$ if $\alpha(\bm{y}, S) >0$. Finally we denote by $C(\bm{y},S)$ the subset of unit vectors of $\bm{1}_d$ that discriminate $\bm{y}$ from $S$.
    \end{definition}
\fi

FROCC has the property that the probability of misclassifying a point outside the class decreases exponentially with parameter $m$. We state this formally. Given $A \subset \bm{1}_d$ we denote the measure of $A$ by $\mu(A)$.
%Note that $\mu(\bm{1}_d)$ is the surface area of the unit sphere in $d$-dimensions and so
%\[\mu(\bm{1}_d) = \frac{2 \pi^{d/2}}{\Gamma(n/2)},\]
%where $\Gamma(x)$ is the gamma function, which is $x!$ if $x$ is a positive integer.

\begin{proposition}
    \label{prp:probability}
    Given a set $S$ and a point $\bm{y}$ with a set $C(\bm{y}, S)$ of non-empty measure such that for any $\bm{w} \in C(\bm{y},S)$, $\langle\bm{y}, \bm{w}\rangle \notin [\min_{j=1}^n \langle \bm{x}_j, \bm{w}\rangle,\max_{j=1}^n \langle \bm{x}_j, \bm{w}\rangle]$, the probability that FROCC projection dimension $m$ returns a Yes answer is
    \[\prob{\bm{y}, S, m} = \left(1 - \frac{\mu(C(\bm{y},S))}{\mu(\bm{1}_d)} \right)^m.\]
    \ifcutlevone
        Hence for the classifier to given an erroneous Yes answer on $\bm{y}$ with probability at most $\delta$ we need
        \begin{equation}
            \label{eq:vanilla-classifier}
            m \geq \frac{\mu(\bm{1}_d)}{\mu(C(\bm{y},S))} \log \frac{1}{\delta}
        \end{equation}
        projection dimensions.
    \fi
\end{proposition}
\begin{figure}[tb]
    \begin{subfigure}{0.32\linewidth}
        \includegraphics[width=1\columnwidth]{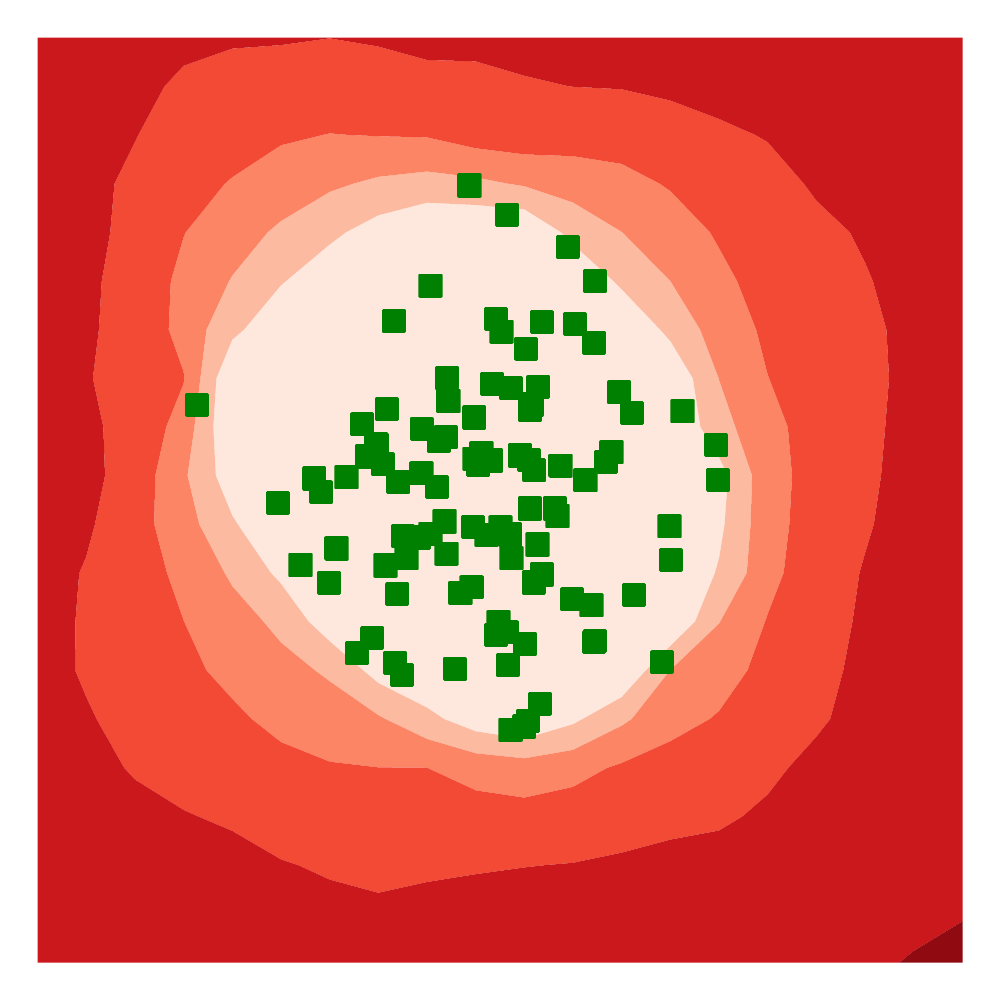}
        \subcaption{m=1}
    \end{subfigure}
    \begin{subfigure}{0.32\linewidth}
        \includegraphics[width=1\columnwidth]{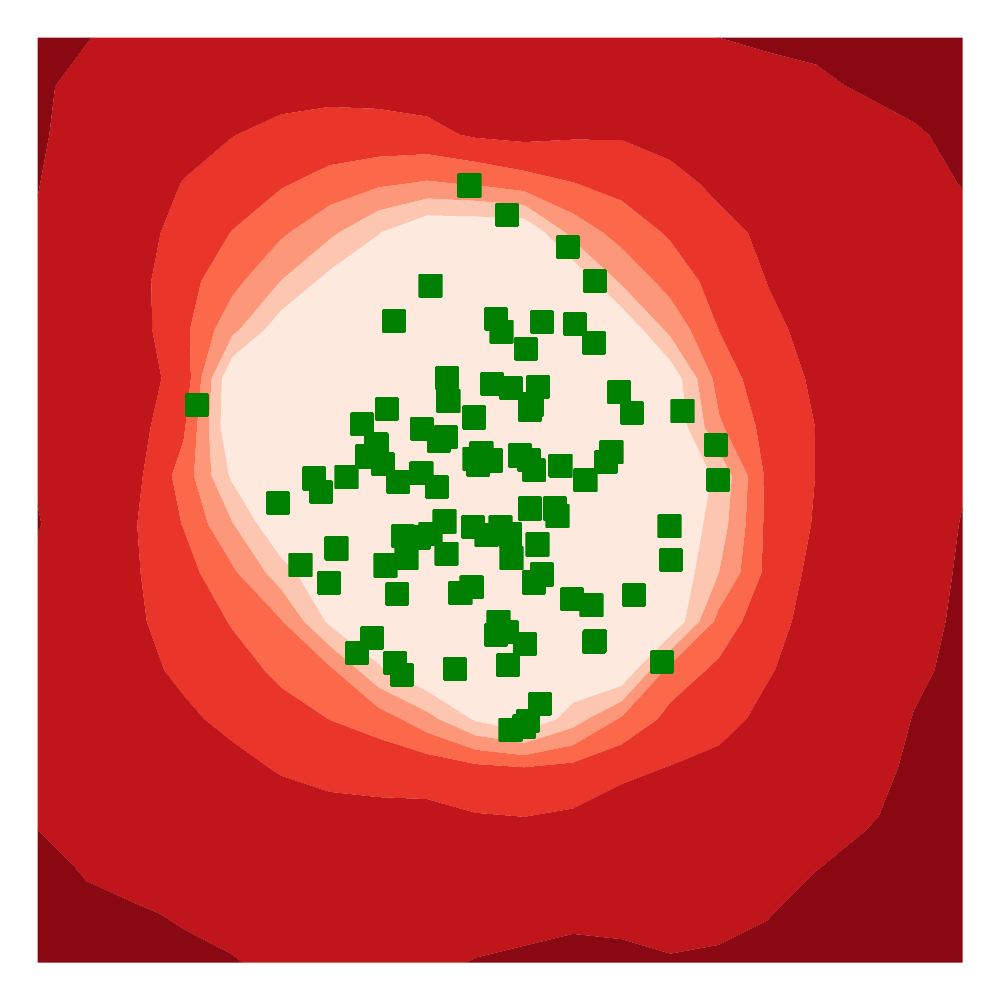}
        \subcaption{m=2}
    \end{subfigure}
    \begin{subfigure}{0.32\linewidth}
        \includegraphics[width=1\columnwidth]{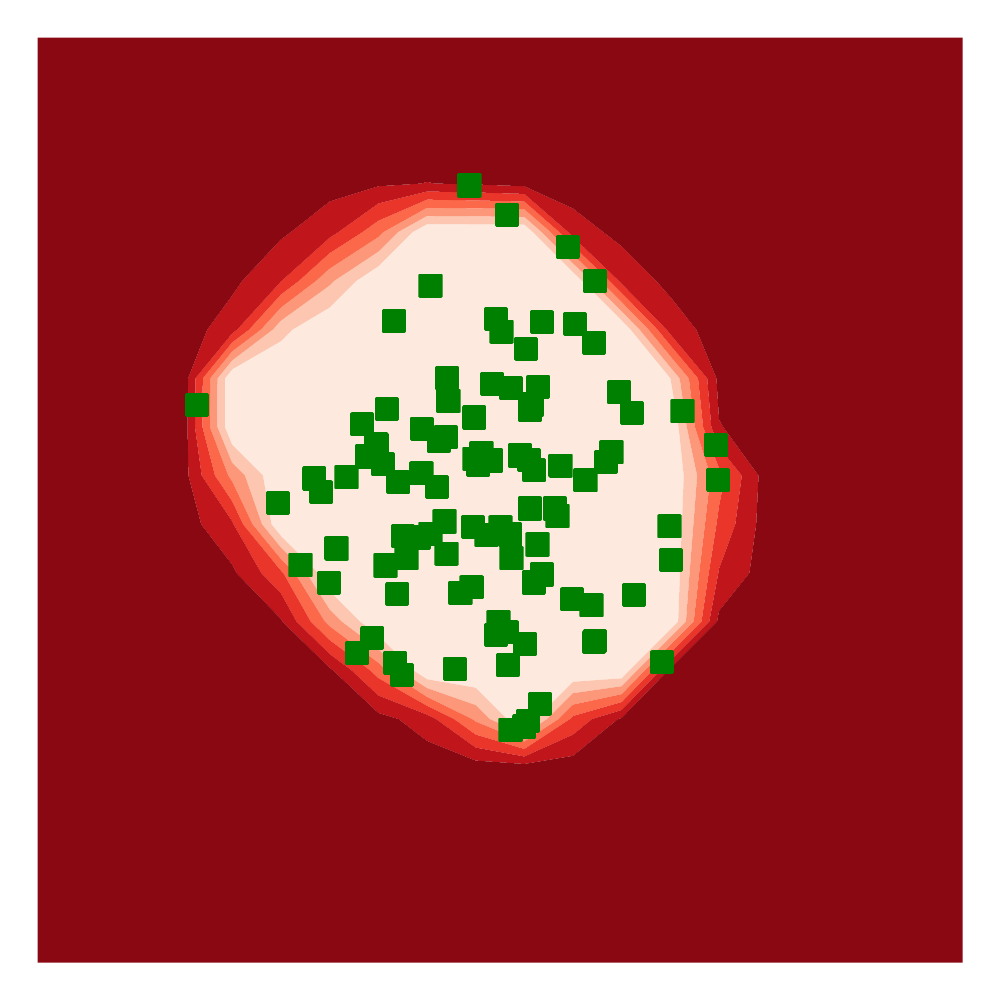}
        \subcaption{m=8}
    \end{subfigure}
    \caption{Probability contours of $\prob{\bm{y}, S, m}$ for different values of $m$. The region boundary gets more distinct with increasing $m$.}
    \label{fig:heatmaps}
\end{figure}
\ifcutlevone
    The proof of~\eqref{eq:vanilla-classifier} is just a simple calculation.
\fi
Figure~\ref{fig:heatmaps} presents the result of Proposition~\ref{prp:probability} visually.
\section{Experiments and Results}
\label{sec:expt}
We conducted a comprehensive set of experiments to evaluate the performance of \frocc on synthetic and real world benchmark datasets and compare them with various existing state of the art baselines. In this section we first present the details of the datasets, baselines, and metrics we used in our empirical evaluation, and then present the results of comparison.

\subsection{Methods Compared}%
\label{sec:methods-compared}
We compare \frocc with the following state of the art algorithms for OCC:

\begin{enumerate}[(i)]
    \item \textbf{OC-SVM:} One class SVM based on the formulation of~\citet{Scholkopf2000},
    \item \textbf{Isolation Forest:} Isolation forest algorithm~\citep{liu2008isolation},
    \item \textbf{AE:} An auto encoder based anomaly detection algorithm based on~\citet{aggarwal2015outlier},
    \item \textbf{Deep SVDD:} A deep neural-network based formulation of~\citet{ruff2018deep}. We used the implementation and the optimal model parameters provided by the authors\footnote{\url{https://github.com/lukasruff/Deep-SVDD-PyTorch}}.
\end{enumerate}

We implemented \frocc using Python 3.8 with \texttt{numpy} library. The IntervalTree subroutine (Algorithm~\ref{alg:interval-tree}) was optimized by quantizing the interval into discrete bins of sizes based on \(\varepsilon\). This allows for a vectorized implementation of the subroutine. All experiments were conducted on a server with Intel Xeon (2.60GHz), and NVIDIA GeForce GTX 1080 Ti GPU. It is noteworthy that unlike deep learning based methods (i.e., AE and Deep SVDD), \frocc does not require GPU during training as well as testing.

For OCSVM and FROCC, for each dataset we chose the best performing (in ROC) kernel among the following: \begin{inparaenum}[(a)] \item linear (dot product), \item radial basis function (RBF), \item polynomial (degree=3) and, \item sigmoid\end{inparaenum}. % moved here from results

\subsection{Datasets}
We evaluate \frocc and the baseline methods on a set of well-known benchmark datasets summarized in Table~\ref{tab:data}. The set of benchmarks we have chosen covers a large spectrum of data characteristics such as the number of samples (from a few hundreds to millions) as well as the number of features (from a handful to tens of thousands). Thus it broadly covers the range of scenarios where OCC task appears.

To adapt the multi-class datasets for evaluation in OCC setting, we perform the processing similar to~\citet{ruff2018deep}: Given a dataset with multiple labels, we assign one class the \emph{positive} label and all others the \emph{negative} label. Specifically, we select the positive classes for i. MNIST: randomly from 0, 1 and 4, ii. CIFAR 10: airplane, automobile and deer and iii. CIFAR 100: beaver, motorcycle and fox. We then split the \emph{positive} class instances into \emph{train} and \emph{test} sets. Using labelled data allows us to evaluate the performance using true labels. We \emph{do not} use any labels for training. We double the test set size by adding equal number of instances from the \emph{negative}ly labeled instances. \ignore{Samples equal to the size of \emph{test} set are added to the \emph{test} set from \emph{negative} set.} We train the OCC on the \emph{train} set which consists of only \emph{positive} examples. Metrics are calculated on \emph{test} set which consists of equal number of \emph{positive} and \emph{negative} examples.%\sbcomment{Please add the reference to the paper which followed the same procedure}

\begin{table}
    \resizebox{\columnwidth}{!}{
        \begin{tabular}{llrrrr}
            \toprule
            ID           & Name                                   & Features & Training Samples & \begin{tabular}{@{}c@{}}Test Samples \\ (Equal +/-- Split)\end{tabular} \\
            \midrule
            A1           & Diabetes~\citep{uci}                   & 8        & 268              & 508                       \\
            A2\(^\star\) & MAGIC Telescope~\citep{uci}            & 10       & 12332            & 20870                     \\
            A3\(^\star\) & MiniBooNE~\citep{uci}                  & 50       & 18250            & 30650                     \\
            A4           & Cardiotocography~\citep{uci}           & 35       & 147              & 294                       \\
            \cmidrule(lr){1-5}
            B1\(^\star\) & SATLog-Vehicle~\citep{vehicle}         & 100      & 24623            & 36864                     \\
            B2\(^\star\) & Kitsune Network Attack~\citep{kitsune} & 115      & 3018972          & 3006490                   \\
            B3           & MNIST~\citep{mnist}                    & 784      & 3457             & 5390                      \\
            B4           & CIFAR-10~\citep{cifar}                 & 3072     & 3000             & 5896                      \\
            B5           & CIFAR-100~\citep{cifar}                & 3072     & 300              & 594                       \\
            B6           & GTSRB~\citep{gtsrb}                    & 3072     & 284              & 562                       \\
            B7           & Omniglot~\citep{omniglot}              & 33075    & 250              & 400                       \\
            \bottomrule
        \end{tabular}
    }
    \caption{Dataset statistics. The datasets are divided in to two groups: Group A (A1:A4) has datasets with fewer than 100 features and Group B (B1:B7) contains datasets with greater than 100 features. \(^\star\) marks the datasets with more than \(10,000\) training samples.}
    \label{tab:data}
\end{table}

\subsection{Metrics}
\noindent\textbf{Classification Performance:} For measuring the classification performance, we use the following two commonly used metrics:
\begin{itemize}
    \item  \textbf{ROC AUC} to measure the classification performance of all methods, and
    \item \textbf{Adjusted Precision@n}, that is, fraction of true outliers among the top n points in the test set, normalized across datasets~\citep{swersky2016evaluation}.
\end{itemize}

\noindent\textbf{Computational Scalability:} For measuring the computational performance of all the methods, we use the speedup of their \textbf{training} and \textbf{testing} time over the fastest baseline method for the dataset. In other words, for an algorithm \(A\) its speedup is calculated as the ratio of training (or testing) time of the fastest \emph{baseline} method to the training (or testing) time of \(A\). We found Isolation forests to be the fastest across all scenarios, thus making it the default choice for computing the speedup metric over. All measurements are reported as an average over 5 runs for each method. %\sbcomment{there is inconsistency of usage - inference, testing, query. Please fix it. I like query in general, but it is not commonly used in ML world.}

\subsection{Results and Observations}
\begin{table}
    \resizebox{\columnwidth}{!}{
        \begin{tabular}{lcccccccc}
            \toprule
            ID           & OC-SVM                    & IsoForest                   & Deep SVDD                   & Auto Encoder              & \frocc                      \\
            \midrule
            A1           & 53.35 \(\pm\) 0.95        & 63.48 \(\pm\) 0.46          & \textbf{73.24 \(\pm\) 0.65} & 62.43 \(\pm\) 0.26        & \emph{73.18 \(\pm\) 0.26}   \\
            A2\(^\star\) & 68.12 \(\pm\) 0.85        & \textbf{77.09 \(\pm\) 0.84} & 61.39 \(\pm\) 0.87          & \emph{75.96 \(\pm\) 0.44} & 72.30 \(\pm\) 0.33          \\
            A3\(^\star\) & 54.22 \(\pm\) 0.45        & \textbf{81.21 \(\pm\) 0.14} & \emph{78.21 \(\pm\) 0.10}   & 56.27 \(\pm\) 0.22        & 69.32 \(\pm\) 0.30          \\
            A4           & 57.93 \(\pm\) 0.73        & \textbf{85.10 \(\pm\) 0.33} & 56.95 \(\pm\) 0.53          & 59.43 \(\pm\) 0.60        & \emph{83.57 \(\pm\) 0.44}   \\
            \midrule
            B1\(^\star\) & 73.51 \(\pm\) 0.41        & \emph{83.19 \(\pm\) 0.44}   & 68.36 \(\pm\) 0.64          & 76.42 \(\pm\) 0.76        & \textbf{85.14 \(\pm\) 0.74} \\
            B2\(^\star\) & 86.01 \(\pm\) 0.60        & 80.05 \(\pm\) 0.43          & \emph{84.73 \(\pm\) 0.37}   & 79.23 \(\pm\) 0.34        & \textbf{86.64 \(\pm\) 0.51} \\
            B3           & 98.99 \(\pm\) 0.11        & 98.86 \(\pm\) 0.31          & 97.04 \(\pm\) 0.08          & \emph{99.38 \(\pm\) 0.13} & \textbf{99.61 \(\pm\) 0.35} \\
            B4           & 52.21 \(\pm\) 0.39        & 51.47 \(\pm\) 0.59          & \emph{59.32 \(\pm\) 0.62}   & 56.73 \(\pm\) 0.47        & \textbf{61.93 \(\pm\) 0.64} \\
            B5           & 49.62 \(\pm\) 0.12        & 54.13 \(\pm\) 0.17          & \emph{67.81 \(\pm\) 0.39}   & 55.73 \(\pm\) 0.53        & \textbf{71.04 \(\pm\) 0.37} \\
            B6           & 63.55 \(\pm\) 0.12        & 61.23 \(\pm\) 0.35          & \textbf{71.88 \(\pm\) 0.36} & 66.73 \(\pm\) 0.44        & \emph{67.12 \(\pm\) 0.60}   \\
            B7           & \emph{90.21 \(\pm\) 0.63} & 70.83 \(\pm\) 0.35          & 72.70 \(\pm\) 0.07          & 62.38 \(\pm\) 0.12        & \textbf{92.46 \(\pm\) 0.21} \\
            \bottomrule
        \end{tabular}
    }
    \caption{Area under the ROC curve. ID corresponds to the dataset ID in Table~\ref{tab:data}. \textbf{Bold} values represent the best scores and \emph{italics} represent the second best scores.}
    \label{tab:roc}
\end{table}
\begin{table}
    \resizebox{\columnwidth}{!}{
        \begin{tabular}{lcccccccc}
            \toprule
            ID           & OC-SVM                      & IsoForest                   & Deep SVDD                   & Auto Encoder              & \frocc                      \\
            \midrule
            A1           & 59.72 \(\pm\) 0.95          & 61.64 \(\pm\) 0.05          & \emph{72.74 \(\pm\) 0.13}   & 63.54 \(\pm\) 0.14        & \textbf{72.83 \(\pm\) 0.27} \\
            A2\(^\star\) & 78.17 \(\pm\) 0.15          & \textbf{87.09 \(\pm\) 0.35} & 72.23 \(\pm\) 0.47          & 78.19 \(\pm\) 0.32        & \emph{84.31 \(\pm\) 0.50}   \\
            A3\(^\star\) & 61.05 \(\pm\) 0.56          & \textbf{79.77 \(\pm\) 0.73} & \emph{78.32 \(\pm\) 0.43}   & 66.17 \(\pm\) 0.48        & 71.91 \(\pm\) 0.36          \\
            A4           & 63.21 \(\pm\) 0.94          & \emph{82.83 \(\pm\) 0.18}   & 60.03 \(\pm\) 0.35          & 57.73 \(\pm\) 0.54        & \textbf{85.91 \(\pm\) 0.33} \\
            \midrule
            B1\(^\star\) & 74.22 \(\pm\) 0.44          & \emph{81.94 \(\pm\) 0.62}   & 68.36 \(\pm\) 0.68          & 76.42 \(\pm\) 0.60        & \textbf{89.21 \(\pm\) 0.35} \\
            B2\(^\star\) & 81.35 \(\pm\) 0.60          & \textbf{87.24 \(\pm\) 0.65} & 81.21 \(\pm\) 0.63          & 72.94 \(\pm\) 0.56        & \emph{83.21 \(\pm\) 0.65}   \\
            B3           & 92.34 \(\pm\) 0.47          & 91.65 \(\pm\) 0.65          & \emph{93.25 \(\pm\) 0.45}   & 74.32 \(\pm\) 0.37        & \textbf{99.80 \(\pm\) 0.58} \\
            B4           & 50.02 \(\pm\) 0.75          & 51.42 \(\pm\) 0.86          & \emph{57.32 \(\pm\) 0.09}   & 51.63 \(\pm\) 0.18        & \textbf{57.92 \(\pm\) 0.41} \\
            B5           & 50.08 \(\pm\) 0.32          & 50.13 \(\pm\) 0.43          & \emph{55.21 \(\pm\) 0.64}   & 53.40 \(\pm\) 0.77        & \textbf{55.62 \(\pm\) 0.38} \\
            B6           & 61.63 \(\pm\) 0.01          & 59.13 \(\pm\) 0.11          & \textbf{69.38 \(\pm\) 0.30} & \emph{64.71 \(\pm\) 0.16} & 62.46 \(\pm\) 0.36          \\
            B7           & \textbf{88.32 \(\pm\) 0.10} & 72.83 \(\pm\) 0.26          & \emph{82.70 \(\pm\) 0.51}   & 65.90 \(\pm\) 0.47        & 79.34 \(\pm\) 0.40          \\
            \bottomrule
        \end{tabular}
    }
    \caption{Adjusted \(Precision@n\). For each dataset, \(n\) is the number of positive examples in the test set of that dataset. ID corresponds to the dataset ID in Table~\ref{tab:data}. \textbf{Bold} values represent the best scores and \emph{italics} represent the second best scores.}
    \label{tab:patn}
\end{table}
\begin{table}
    \resizebox{\columnwidth}{!}{
        \begin{tabular}{lcccccccc}
            \toprule
            \multirow{2}{*}{ID} & \multicolumn{2}{c}{OC-SVM} & \multicolumn{2}{c}{Deep SVDD} & \multicolumn{2}{c}{Auto Encoder} & \multicolumn{2}{c}{\frocc}                                                                        \\
            \cmidrule(lr){2-3}                \cmidrule(lr){4-5}                \cmidrule(lr){6-7}                   \cmidrule(lr){8-9}
                                & Train                      & Test                          & Train                            & Test                       & Train & Test  & Train                               & Test           \\
            \midrule
            A1                  & 0.055                      & 0.862                         & 0.034                            & 0.038                      & 0.089 & 0.090 & \textbf{1.111}\textsuperscript{(2)} & \textbf{1.111} \\
            A2\(^\star\)        & 0.057                      & 0.847                         & 0.032                            & 0.038                      & 0.085 & 0.088 & \textbf{1.333}\textsuperscript{(3)} & \textbf{1.075} \\
            A3\(^\star\)        & 0.100                      & 0.840                         & 0.031                            & 0.039                      & 0.082 & 0.087 & \textbf{1.282}\textsuperscript{(3)} & \textbf{1.075} \\
            A4                  & 0.064                      & 0.870                         & 0.038                            & 0.038                      & 0.091 & 0.091 & \textbf{1.190}\textsuperscript{(2)} & \textbf{1.099} \\
            \midrule
            B1\(^\star\)        & 0.074                      & 0.855                         & 0.032                            & 0.038                      & 0.087 & 0.089 & \textbf{1.075}\textsuperscript{(1)} & \textbf{1.087} \\
            B2\(^\star\)        & 0.043                      & 0.524                         & 0.028                            & 0.038                      & 0.062 & 0.078 & \textbf{1.786}\textsuperscript{(1)} & \textbf{1.176} \\
            B3                  & 0.099                      & 0.794                         & 0.043                            & 0.039                      & 0.050 & 0.080 & \textbf{1.299}\textsuperscript{(1)} & \textbf{1.149} \\
            B4                  & 0.099                      & 0.820                         & 0.040                            & 0.039                      & 0.069 & 0.083 & \textbf{1.389}\textsuperscript{(1)} & \textbf{1.099} \\
            B5                  & 0.100                      & 0.826                         & 0.039                            & 0.039                      & 0.074 & 0.085 & \textbf{1.370}\textsuperscript{(1)} & \textbf{1.087} \\
            B6                  & 0.098                      & 0.833                         & 0.039                            & 0.039                      & 0.079 & 0.086 & \textbf{1.351}\textsuperscript{(2)} & \textbf{1.087} \\
            B7                  & 0.100                      & 0.806                         & 0.035                            & 0.039                      & 0.061 & 0.082 & \textbf{1.449}\textsuperscript{(1)} & \textbf{1.111} \\
            \bottomrule
        \end{tabular}
    }
    \caption{Training and test speedups for baselines and \frocc. It is calculated as \(speedup(M) = \frac{time(IF)}{time(M)}\), where \(time(IF)\) is the time for Isolation Forest, the fastest baseline. Higher is better. A number greater than 1 indicates the method is faster than the fastest baseline. Superscript shows the rank of the methods based on ROC-AUC scores for \frocc.}
    \label{tab:speedup}
\end{table}

Table~\ref{tab:roc} compares the ROC scores of the baselines on the datasets. From the table we note that \frocc outperforms the baselines in six of the eleven datasets. Further, we note that for high dimensional datasets --i.e., datasets with \emph{\(\geq\) 100 features} (B1-B7), \frocc outperforms the baselines in \textbf{six out of seven datasets while ranking second in the remaining one}. This demonstrates the superiority of \frocc for high dimensional data which are known to be challenging for most OCC methods.  Table~\ref{tab:patn} summarizes the \emph{Adjusted Precision@n} (for a given dataset, \(n\) is the number of positive examples in test data). %in addition to \emph{area under the ROC curve} to provide a more complete picture of the performance.
Here too we observe that \frocc outperforms the baselines in six datasets. In high dimensional datasets, it is superior in four out of seven datasets, and is the second best in one. And even in datasets where the performance of \frocc is not the highest, it does not drop as low as traditional shallow baselines, which in two occasions fail to learn altogether.

Next, we turn our attention to the computational performance of all the methods summarized in Table~\ref{tab:speedup}, where we present the training and testing speedups. These results particularly highlight the  scalability of \frocc over all other baselines, including the traditionally efficient shallow methods such as Isolation forest. Against deep learning based methods such as Deep SVDD and Auto Encoder, \frocc shows nearly \emph{1-2 orders} magnitude speedup. Overall, we notice that \frocc achieves up to \(68 \times\) speedup while beating the baselines, and up to \(35 \times\) speedup while ranking second or third. %\sbcomment{I suggest you structure this discussion better: first, explain what each table is showing, and what insights individual table provides. Then we can get into discussion points.}%

\subsection{Discussion of Results} 
Here we discuss the results from two broad perspectives:
\paragraph{ROC and Speed trade-offs:} The speedups show that \frocc is the fastest method. From Table~\ref{tab:roc}, we see that in six datasets, \frocc outperforms the baseline as well. Of the cases it doesn't outperform the baselines, it ranks second thrice out of five and in each of the these cases (datasets A1, A4 and B7) it is 35, 29 and 1.19 times faster than the best performing method respectively. Thus, depending on the application, the speed vs ROC trade-off may make \frocc the method of choice even when it doesn't outperform other algorithms.

\paragraph{Scalability:} While we are faster than baselines in all datasets, large datasets like MiniBooNE (A3), SATLog-Vehicle (B1) and especially Kitsune network attack (B2) datasets -- where \frocc is 68 times faster than the second best method -- can be render the best performing deep baselines infeasible, leaving scalable methods like IsolationForest and \frocc as the only option. Given that \frocc consistently ranks among top three (1: 6 times, 2: 3 times, 3: 2 times), whereas IsolationForest's rank varies inconsistently and go up to 5 (out of 5), \frocc is the algorithm of choice when scalability is a concern.

% To summarize, our experimental observations:
% \begin{enumerate}
%     \item We outperform the baselines in six out of ten datasets in both ROC AUC and adjusted \(precision@n\), while remaining competitive in other datasets.
%     \item We are 2 orders of magnitude faster than the deep benchmarks.
%     \item Our performance doesn't drop as low as non-deep benchmarks, which, in two occasions fail to learn at all.
% \end{enumerate}

\subsection{Parameter Sensitivity}
In previous section we mentioned that the parameters were chosen based on grid search. In this section we present a study of sensitivity of \frocc to the parameters.

\paragraph{Sensitivity to dimensions.} From the analysis in Section~\ref{sec:theory}, we know that the performance of \frocc improves with increasing number of classification dimensions. Figure~\ref{fig:dim-sens} shows this variation for a few selected datasets.

\begin{figure}[htb]
    \begin{subfigure}{0.45\linewidth}
        \includegraphics[width=1\columnwidth]{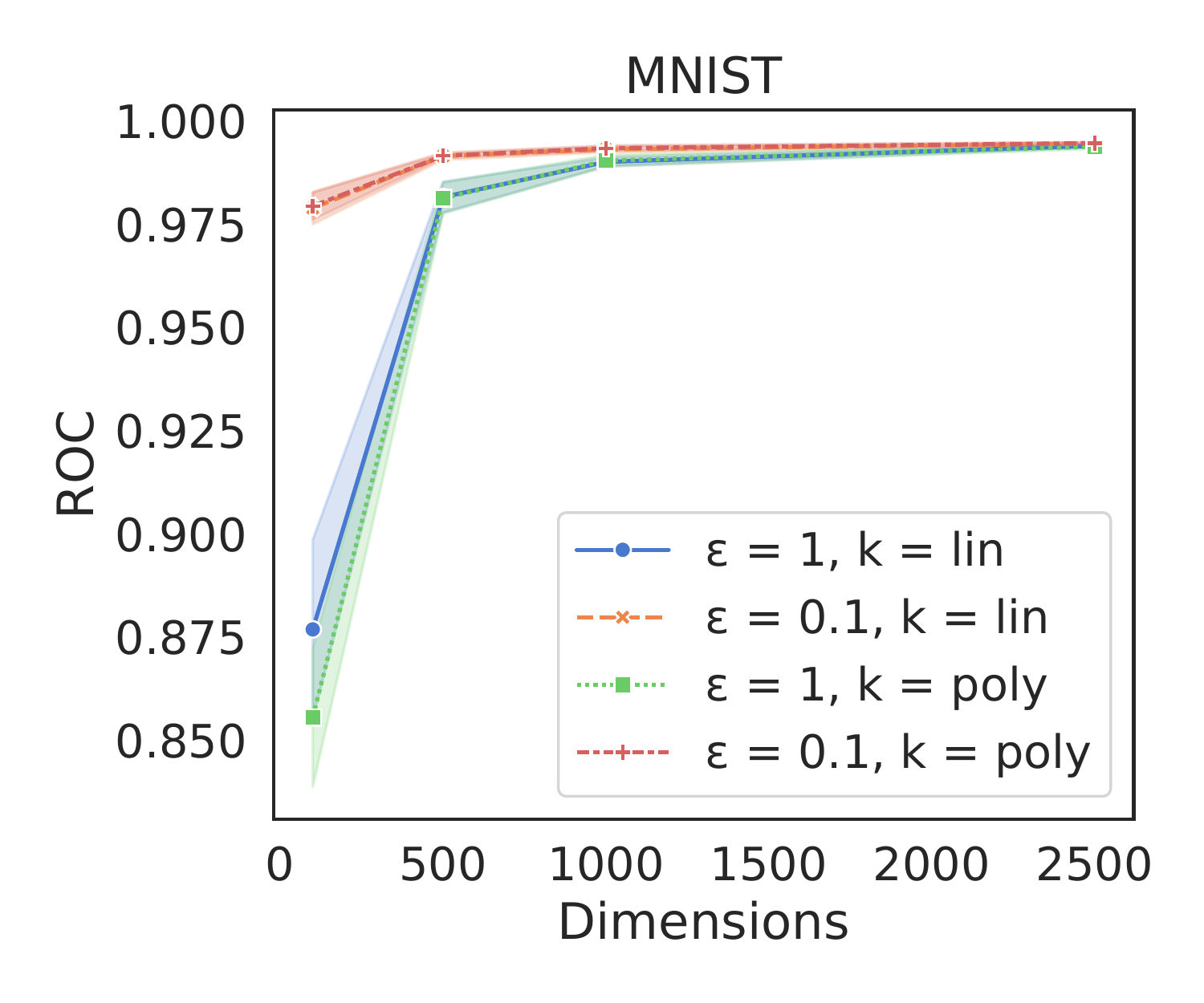}
    \end{subfigure}
    \begin{subfigure}{0.45\linewidth}
        \includegraphics[width=1\columnwidth]{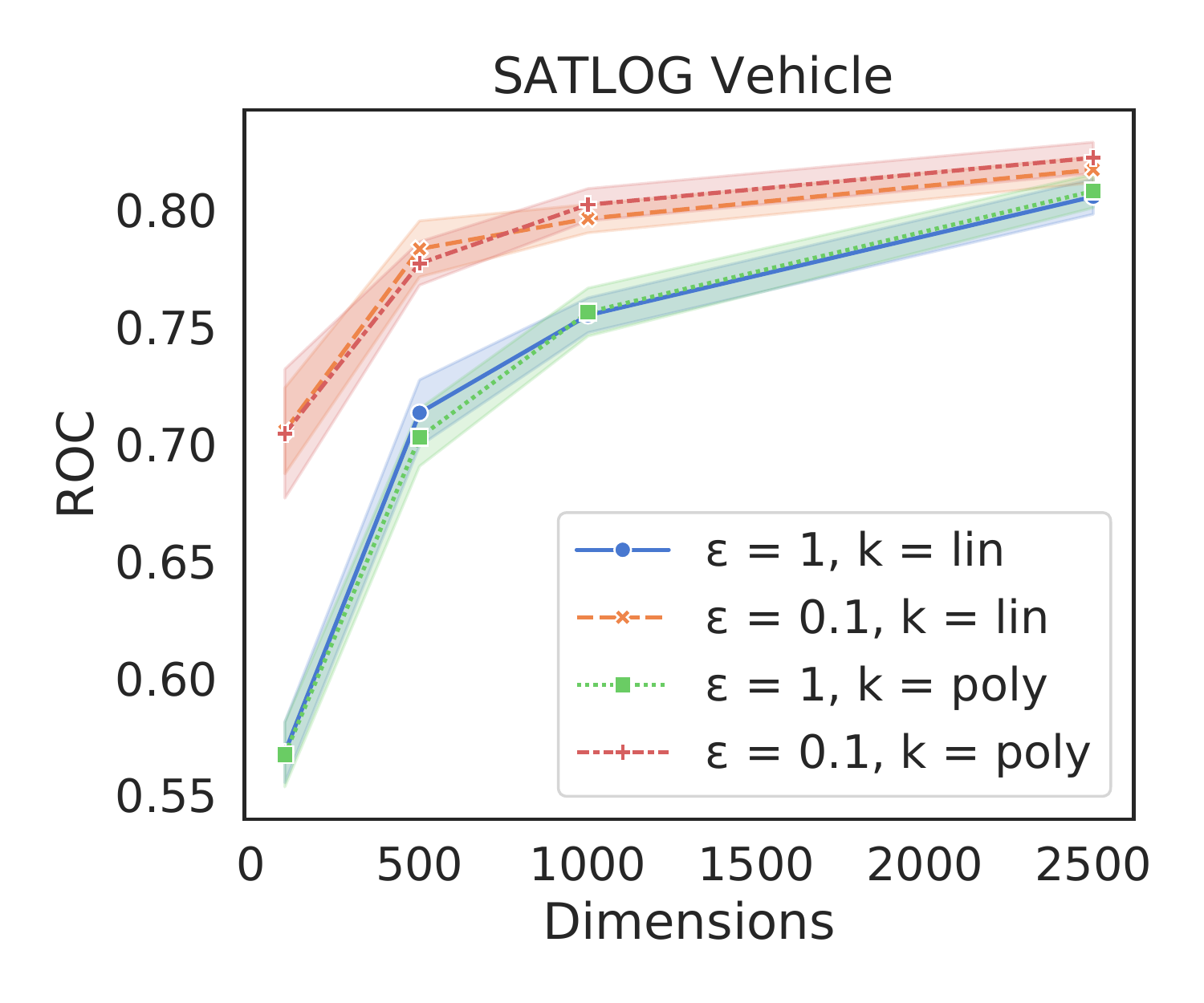}
    \end{subfigure}

    \begin{subfigure}{0.45\linewidth}
        \includegraphics[width=1\columnwidth]{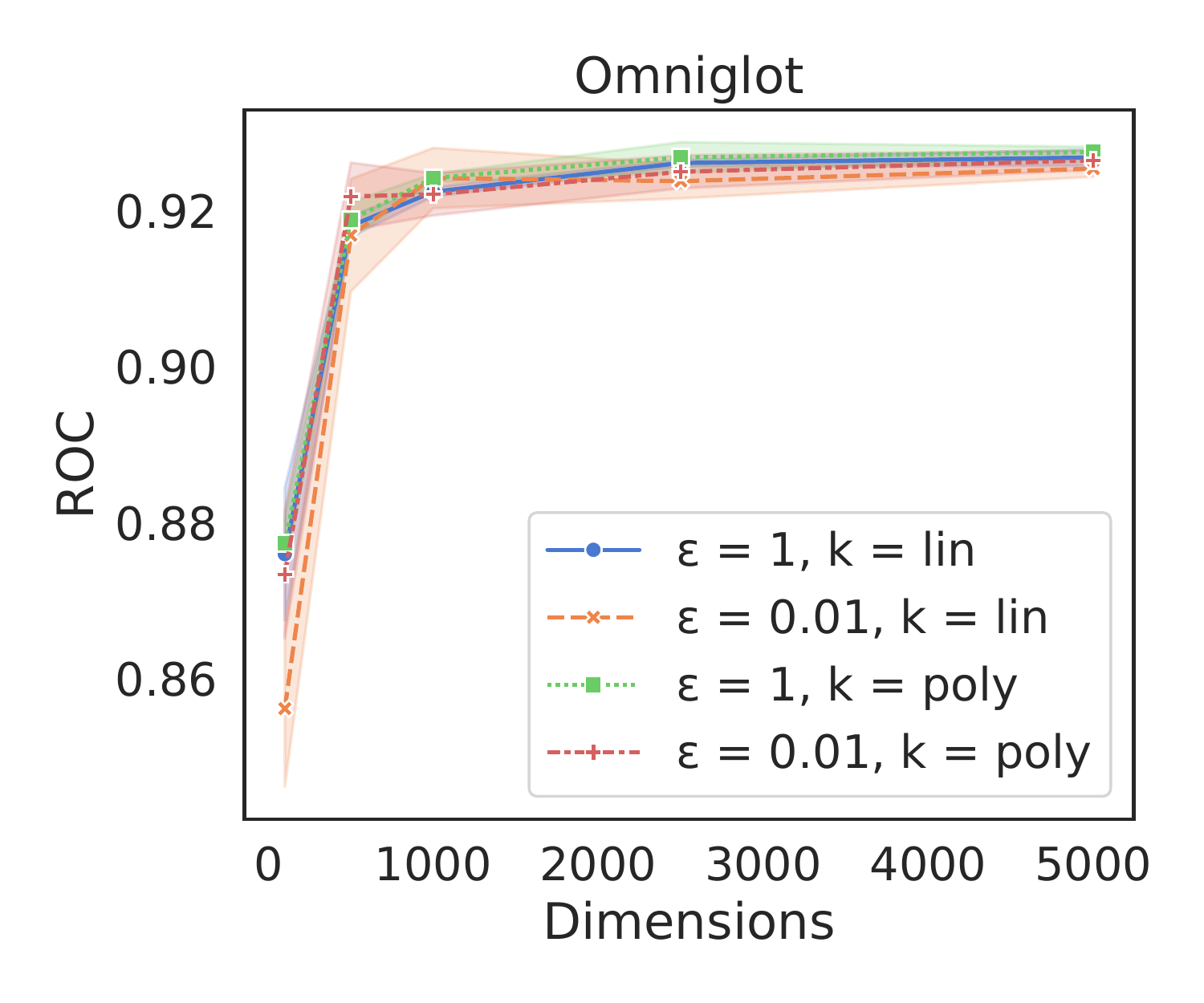}
    \end{subfigure}
    \begin{subfigure}{0.45\linewidth}
        \includegraphics[width=1\columnwidth]{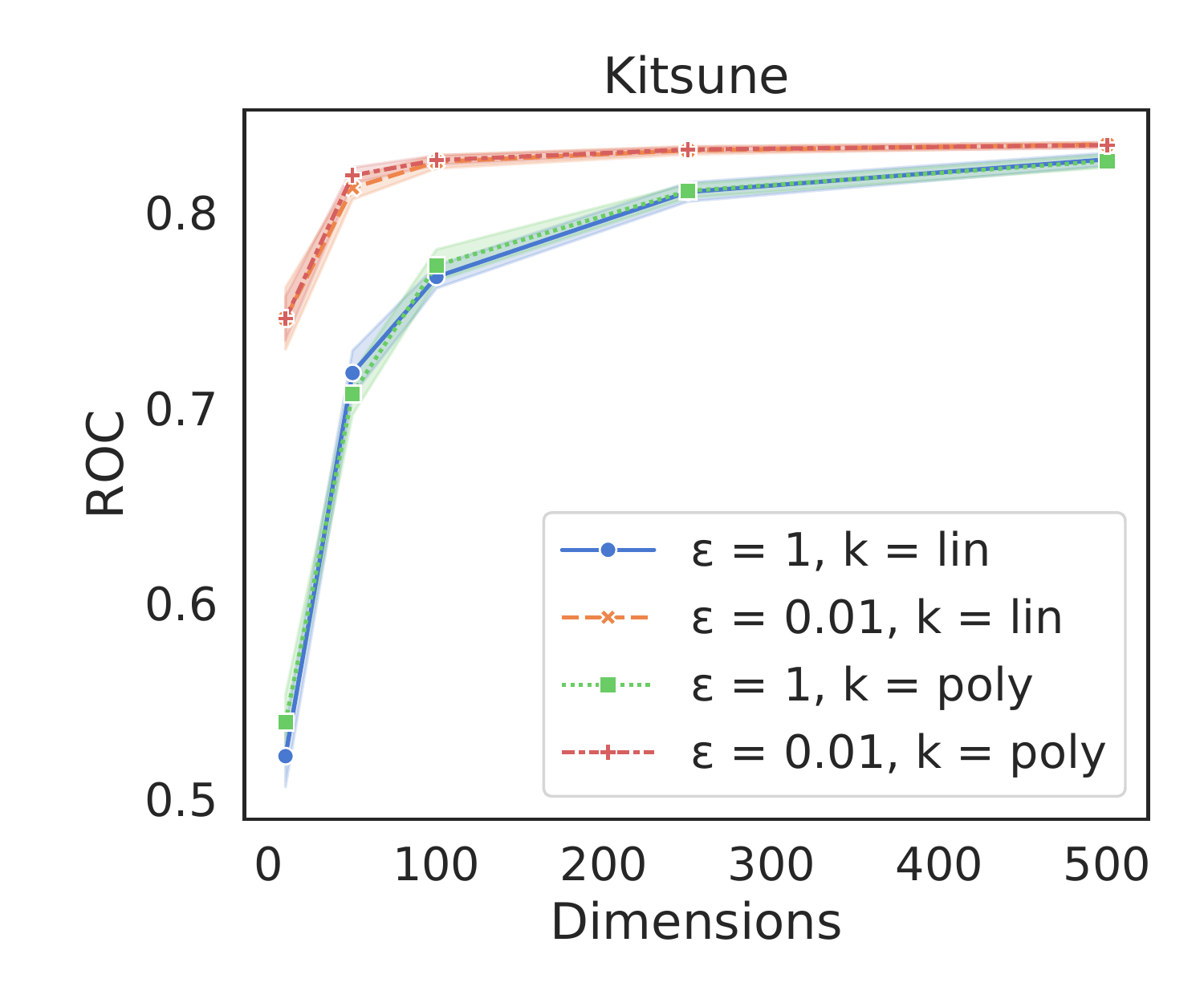}
    \end{subfigure}
    \caption{Variation of area under the ROC with Dimensions for MNIST, SATLOG Vehicle, Omniglot and Kitsune network attack dataset. We observe that the metric improves monotonically, confirming the theoretical claim.}
    \label{fig:dim-sens}
\end{figure}

\paragraph{Sensitivity to \(\varepsilon\).} The separation parameter \(\varepsilon\) is a dataset specific parameter and relates to the minimum distance of separation between positive examples and potential negative examples. In practice this comes out to be around the order of 0.1 or 0.01, as shown in Figure~\ref{fig:eps-sens}

\begin{figure}[htb]
    \begin{subfigure}{0.45\linewidth}
        \includegraphics[width=1\columnwidth]{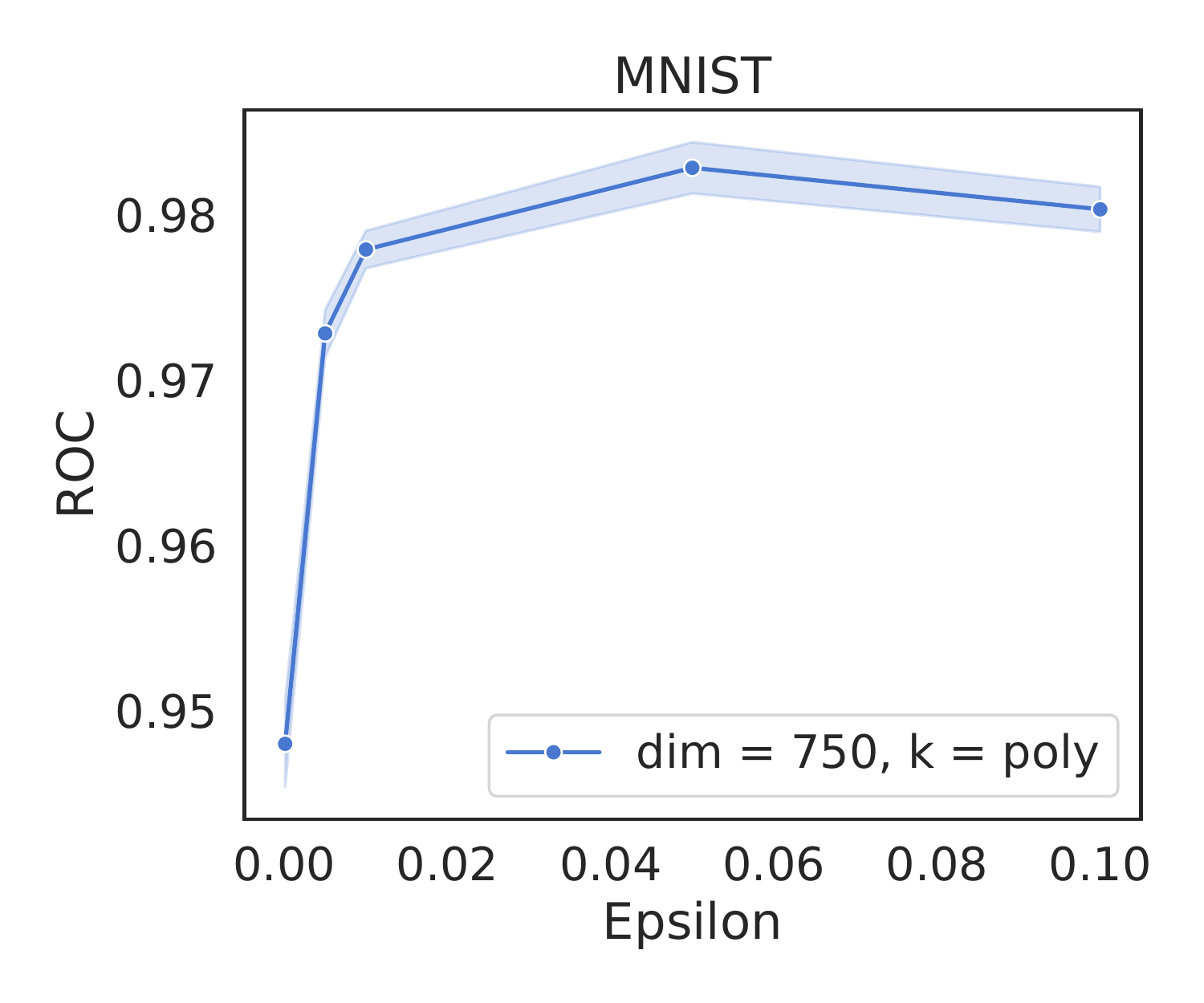}
    \end{subfigure}
    \begin{subfigure}{0.45\linewidth}
        \includegraphics[width=1\columnwidth]{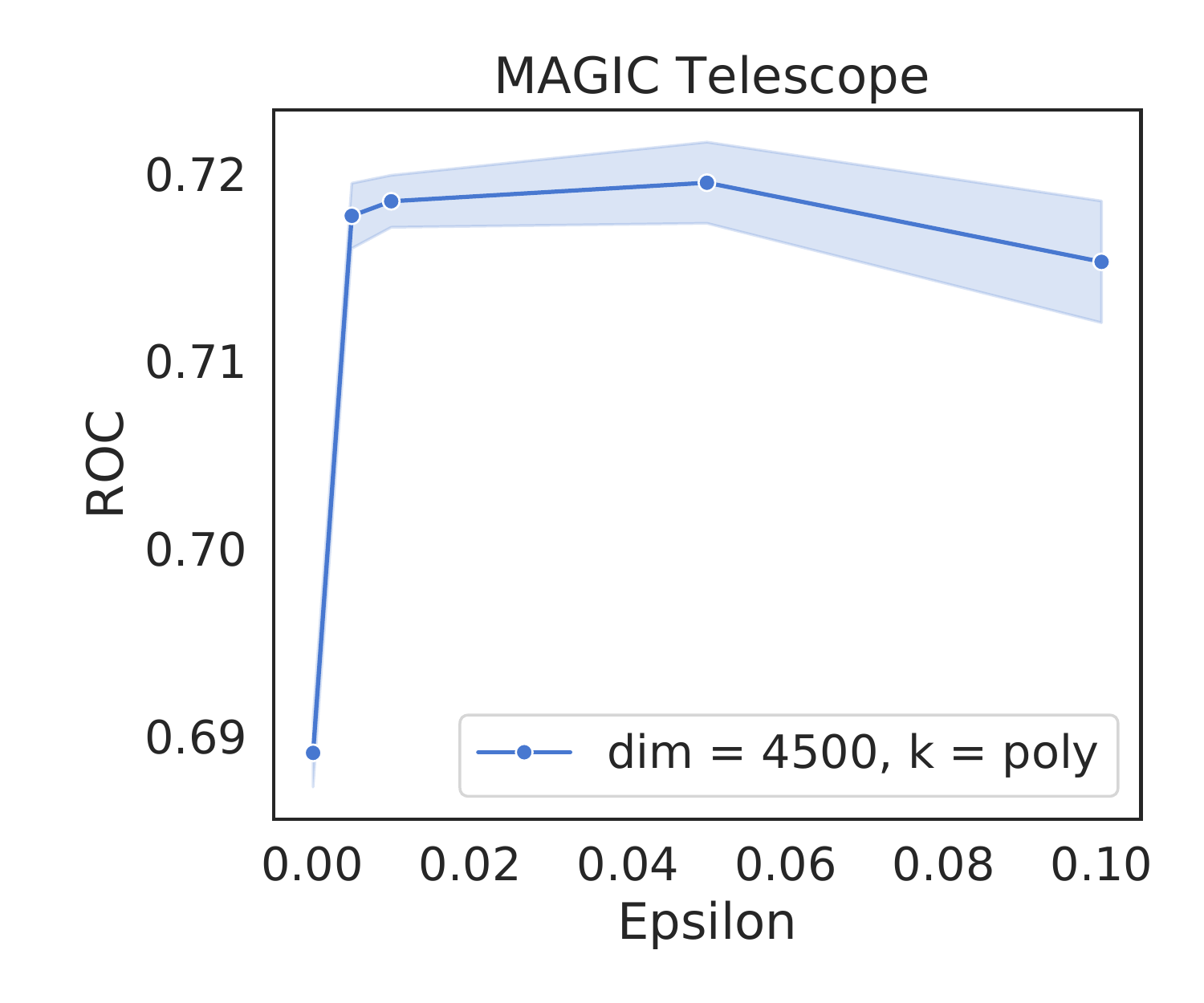}
    \end{subfigure}
    \caption{Variation of area under the ROC with \(\varepsilon\) for MNIST and MAGIC Telescope dataset.}
    \label{fig:eps-sens}
\end{figure}
\section{Conclusion}
\label{sec:conclusion}
In this paper, we presented a surprisingly simple and highly effective one-class classification strategy based on projecting data onto randomly selected vectors from the unit sphere. By representing these projections as a collections of intervals we developed \frocc. Our experiments over real and synthetic datasets reveal that despite its extreme simplicity --and the resulting computational efficiency-- \frocc is not just \emph{better} than the traditional methods for one-class classification but also better than the highly sophisticated (and expensive) models that use deep neural networks.

As part of our future work, we plan to explore extending the \frocc algorithms to the classification task under highly imbalanced data, and to application settings where high-performance OCC are crucial.

% \input{sections/declaration.tex}
% \begin{acknowledgements}
%     The authors thank the IBM AI Horizons Network and Nutanix, Inc. for their valuable insights.
% \end{acknowledgements}

\bibliographystyle{spbasic}
\bibliography{main}

\end{document}